%% file: main.tex
\documentclass{article}

\usepackage[utf8]{inputenc} %
\usepackage[T1]{fontenc}    %
\usepackage{hyperref}       %
\usepackage{url}            %
\usepackage{booktabs}       %
\usepackage{amsfonts}       %
\usepackage{nicefrac}       %
\usepackage{microtype}      %
\usepackage{graphicx}
\usepackage{subcaption}
\usepackage{multirow}
\usepackage{tabularx}
\usepackage{amsmath}
\usepackage{amsthm}
\usepackage{wrapfig}
\usepackage[nocompress]{cite}
\usepackage{cancel}
\usepackage{footnote}
\makesavenoteenv{tabular}
\makesavenoteenv{table}
\usepackage{thmtools}
\usepackage{thm-restate}
\usepackage{hyperref}
\usepackage[capitalise]{cleveref}
\usepackage{color,xcolor}
\usepackage{etoolbox}

\newif\ifcomments
\commentsfalse
\ifcomments
\newcommand\KG[1]{\textcolor{magenta}{[Karan: #1]}}
\newcommand\AG[1]{\textcolor{red}{[Albert: #1]}}
\newcommand\SL[1]{\textcolor{blue}{#1}}

\newcommand\CR[1]{\textcolor{red}{[CR: #1]}}
\else
\newcommand\KG[1]{}
\newcommand\AG[1]{}
\newcommand\SL[1]{}
\newcommand\CR[1]{}
\fi

\newcommand\hidden[1]{}

\newcommand{\methodabbrv}{CAMEL}

\DeclareMathOperator*{\argmax}{arg\,max}

\newcommand{\E}{\mathbb{E}}
\newcommand{\R}{\mathbb{R}}
\newcommand{\mt}{\tilde{m}}
\newcommand{\xt}{\tilde{x}}

\newcommand{\Yh}{\hat{Y}}
\newcommand{\at}{\tilde{a}}

\newcommand{\tA}{\tilde{A}}
\newcommand{\tB}{\tilde{B}}

\newcommand\blfootnote[1]{%
  \begingroup
  \renewcommand\thefootnote{}\footnote{#1}%
  \addtocounter{footnote}{-1}%
  \endgroup
}

\newtheorem{proposition}{Proposition}
\newtheorem{definition}{Definition}

\newtheorem{corollary}{Corollary}
\newtheorem{lemma}{Lemma}

\declaretheorem[name=Definition]{defn}

\graphicspath{{figs/}}

\title{Model Patching: Closing the Subgroup Performance Gap with Data Augmentation}

\newtoggle{arxiv}
\toggletrue{arxiv}

\iftoggle{arxiv}
{
\usepackage[square,sort,numbers]{natbib}
\setlength{\textwidth}{6.5in}
\setlength{\textheight}{9in}
\setlength{\oddsidemargin}{0in}
\setlength{\evensidemargin}{0in}
\setlength{\topmargin}{-0.5in}
\newlength{\defbaselineskip}
\setlength{\defbaselineskip}{\baselineskip}
\setlength{\marginparwidth}{0.8in}
}
{
\usepackage[nonatbib]{neurips_2020}
}

\iftoggle{arxiv}{
\usepackage{authblk}
\author[$*$]{Karan Goel}
\author[$*$]{Albert Gu}
\author{Yixuan Li}
\author{Christopher R{\'e}}
\affil[]{Department of Computer Science, Stanford University\vspace{4pt}}
\affil[ ]{\texttt{\{krng,albertgu\}@stanford.edu}, \texttt{\{sharonli,chrismre\}@cs.stanford.edu}}
}
{
\usepackage[nonatbib]{neurips_2020}

\author{%
  Karan Goel$^*$,
  Albert Gu$^*$,
  Sharon Li,
  Christopher R\'{e} \\
  Department of Computer Science, Stanford University \\
  \texttt{\{krng,albertgu\}@stanford.edu},
  \texttt{\{sharonli,chrismre\}@cs.stanford.edu}
}
}

\begin{document}

\maketitle

\begin{abstract}
Classifiers in machine learning are often brittle when deployed. 
Particularly concerning are models with inconsistent performance on specific \emph{subgroups} of a class, 
e.g., exhibiting disparities in skin cancer classification in the presence or absence of a spurious bandage.
To mitigate these performance differences, 
we introduce \emph{model patching}, 
a two-stage framework for improving robustness that encourages the model to be invariant to subgroup differences, and focus on class information shared by subgroups.
Model patching
first models subgroup features within a class and learns semantic transformations between them,
and then trains a classifier with data augmentations that deliberately manipulate subgroup features.
We instantiate model patching with {\methodabbrv}, which (1) uses a CycleGAN to learn the intra-class, inter-subgroup augmentations, and (2) balances subgroup performance using a theoretically-motivated subgroup consistency regularizer, accompanied by a new robust objective.
We demonstrate \methodabbrv's effectiveness on 3 benchmark datasets, with reductions in robust error of up to 33\% relative to the best baseline. Lastly, \methodabbrv~successfully patches a model that fails due to spurious features on a real-world skin cancer dataset.\blfootnote{Code for Model Patching can be found at \url{https://github.com/HazyResearch/model-patching}.}
\end{abstract}
\input{src/intro.tex}

\input{src/method.tex}

\input{src/analysis.tex}

\input{src/experiments.tex}

\input{src/conclusion.tex}
\label{submission}

\nocite{he2016identity,cubuk2019autoaugment,ho2019population,lim2019fast}
\nocite{ratner2017learning,devries2017improved,yun2019cutmix,berthelot2019mixmatch,zhang2017mixup,upchurch2017deep,isola2017image,zhu2017unpaired,almahairi2018augmented,choi2018stargan,antoniou2017data,gowal2019achieving,heinze2017conditional,kannan2018adversarial,zheng2016improving}

\section*{Broader Impact}
Model patching addresses an important problem faced by domain experts: the unexpected failure of standard classifiers on subgroups of a class. This failure can have important consequences in real applications such as inducing discrimination and bias toward certain subgroups or populations. As an illustrative example, consider that skin cancer image classification datasets overwhelmingly contain images of light-skinned individuals~\cite{10.1001/jamadermatol.2018.2348}, suggesting that performance on underrepresented subgroups corresponding to darker skin tones may suffer when a model trained on these datasets is deployed. Through this work and by releasing our code, we hope to both provide more clarity on the methodological question of how to make such models better, as well as giving domain experts a new tool that takes an encouraging step in this direction. While we do not anticipate any negative consequences to our work, we hope to continue to improve and build on model patching in future work.  %

\section*{Acknowledgments and Disclosure of Funding}

We thank Pang Wei Koh, Shiori Sagawa, Geoff Angus, Jared Dunnmon, and Nimit Sohoni for assistance with baselines and datasets and useful discussions.
We thank members of the Hazy Research group including Mayee Chen, Megan Leszczynski, Sarah Hooper, Laurel Orr, and Sen Wu for useful feedback on previous drafts.
KG and AG are grateful for Sofi Tukker's assistance throughout this project.
We gratefully acknowledge the support of DARPA under Nos. FA86501827865 (SDH) and FA86501827882 (ASED); NIH under No. U54EB020405 (Mobilize), NSF under Nos. CCF1763315 (Beyond Sparsity), CCF1563078 (Volume to Velocity), and 1937301 (RTML); ONR under No. N000141712266 (Unifying Weak Supervision); the Moore Foundation, NXP, Xilinx, LETI-CEA, Intel, IBM, Microsoft, NEC, Toshiba, TSMC, ARM, Hitachi, BASF, Accenture, Ericsson, Qualcomm, Analog Devices, the Okawa Foundation, American Family Insurance, Google Cloud, Swiss Re, the Salesforce Deep Learning Research grant, the HAI-AWS Cloud Credits for Research program, and members of the Stanford DAWN project: Teradata, Facebook, Google, Ant Financial, NEC, VMWare, and Infosys. The U.S. Government is authorized to reproduce and distribute reprints for Governmental purposes notwithstanding any copyright notation thereon. Any opinions, findings, and conclusions or recommendations expressed in this material are those of the authors and do not necessarily reflect the views, policies, or endorsements, either expressed or implied, of DARPA, NIH, ONR, or the U.S. Government.

\small

\bibliographystyle{abbrv}
\bibliography{biblio}

\clearpage

\appendix

\input{src/glossary.tex}
\input{src/relatedwork_full.tex}

\input{src/analysis_full.tex}

\input{src/experiment_details.tex}

\end{document}

%% file: src/intro.tex
\section{Introduction} 
\label{sec:intro}
Machine learning models typically optimize for average performance, 
and when deployed, can yield inaccurate predictions on important subgroups of a class.
For example, practitioners have noted that on the ISIC skin cancer detection dataset~\cite{codella2018skin}, classifiers are more accurate on images of benign skin lesions with visible bandages, when compared to benign images where no bandage is present~\cite{Bissoto2019DeCB, Rieger2019InterpretationsAU}. 

\begin{wrapfigure}{R}{0.45\textwidth}%
    \centering
    \includegraphics[width=0.95\linewidth]{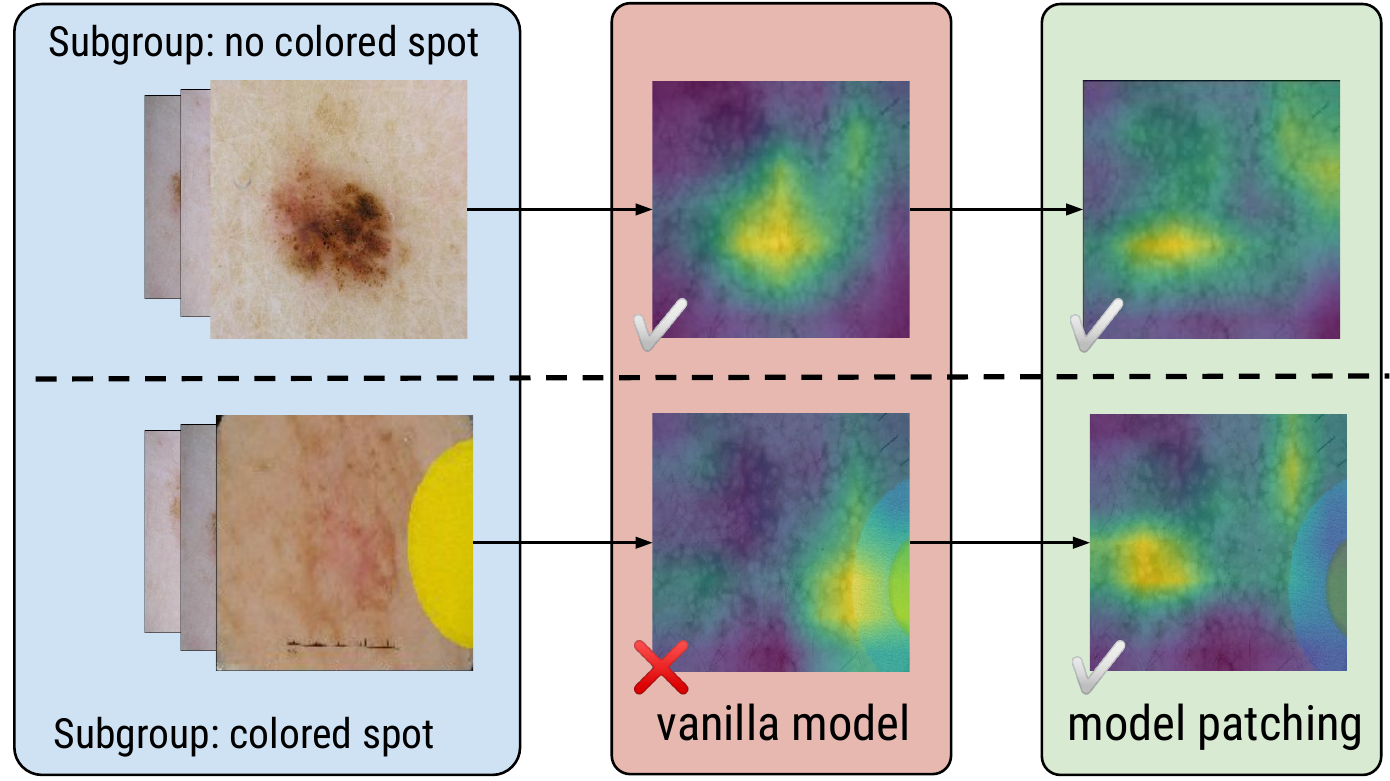}
    \caption{\small A vanilla model trained on a skin cancer dataset exhibits a {subgroup performance gap} between images of malignant cancers with and without colored bandages. GradCAM~\cite{selvaraju2017grad} illustrates that the vanilla model spuriously associates the colored spot with {benign} skin lesions. With model patching, the malignancy is predicted correctly for both subgroups.}%
    \label{fig:isic}
\end{wrapfigure}

This {subgroup performance gap} is an undesirable consequence of a classifier's reliance 
on subgroup-specific features, e.g. spuriously associating colorful bandages with a benign cancer class (Figure~\ref{fig:isic}). 
A common strategy to side-step this issue is to use manual data augmentation to erase the differences between subgroups, e.g., using Photoshop~\cite{winkler2019association} or image processing tools~\cite{Rieger2019InterpretationsAU} to remove markings on skin cancer data before retraining a classifier. 
However, hand-crafting these augmentations may be impossible if the subgroup differences are difficult to express.

Ideally, we would automatically learn the features differentiating the subgroups of a class, and then encourage a classifier to be invariant to these features when making its prediction. To this end, we introduce \emph{model patching}, a framework that encapsulates this solution in two stages:
\begin{itemize}
    \item {\it Learn inter-subgroup transformations.} Isolate features that differentiate subgroups within a class, learning inter-subgroup transformations
    between them. 
    These transformations change an example's subgroup identity but preserve the class label.
    \item {\it Train to patch the model.} Leverage the transformations as controlled data augmentations that manipulate subgroup features, encouraging the classifier to be robust to their variation. 
\end{itemize}

In the first stage of model patching (Section~\ref{sec:cyclegan}), we learn, rather than specify, the differences between the subgroups of a class. Our key insight here is to learn these differences as {inter-subgroup transformations} 
that modify the subgroup membership of examples, while preserving class membership. %
Applying these semantic transformations as data augmentations in the second stage allows us to generate ``imagined'' versions of an example in the other subgroups of its class. 
This contrasts with conventional data augmentation, where heuristics such as rotations, flips, MixUp or CutOut~\cite{zhang2017mixup, devries2017improved} are hand-crafted rather than learned. 
While these heuristics have been shown to improve robustness~\cite{hendrycks2019augmix}, the invariances they target are not well understood. %
Even when augmentations are learned~\cite{ratner2017learning}, they are used to address data scarcity, rather than manipulate examples to improve robustness in a prescribed way. 
Model patching is the first framework for data augmentation that directly targets subgroup robustness. 

The goal of the second stage (Section~\ref{sec:stage2}) is to appropriately use the transformations to remove the classifier's dependence on subgroup-specific features. We introduce two algorithmic innovations that target subgroup robustness: 
(i) a subgroup robust objective and; 
(ii) a subgroup consistency regularizer. 
Our subgroup robust objective extends prior work on group robustness \cite{sagawa2019distributionally} to our subgroup setting, where classes and subgroups form a hierarchy (Figure~\ref{fig:splash} left).
Our new subgroup consistency regularizer constrains the predictions on original and augmented examples to be similar. While recent work on consistency training~\cite{hendrycks2019augmix, xie2019unsupervised} has been empirically successful in constructing models that are robust to perturbations, our consistency loss carries theoretical guarantees on the model's robustness.
We note that our changes are easy to add on top of standard classifier training.

We contribute a theoretical analysis (Section~\ref{sec:invariance-analysis}) to motivate our end-to-end framework. Our analysis codifies the distributional assumptions underlying the class-subgroup hierarchy and motivates our new consistency regularizer, which has a simple information theoretic interpretation under this framework.
First, we introduce a natural model for the data generating process that decouples an example from its subgroup. %
Under this model, the mutual information between the subgroup information carried by the data and the classifier's output is related to a particular Jensen-Shannon divergence that is captured by our subgroup consistency loss.
This enables us to prove that our consistency loss, when applied to subgroup-augmented examples from the first stage, directly bounds a mutual information objective capturing the subgroup-invariance of the trained classifier. Thus, training with our end-to-end framework forces the classifier to be invariant to subgroup-specific features.

We conduct an extensive empirical study (Section~\ref{sec:experiments}) that validates CycleGAN Augmented Model Patching (\methodabbrv)'s ability to improve subgroup invariance and robustness.
We first evaluate \methodabbrv~on a controlled MNIST setup, where it cuts robust error rate to a third of other approaches while learning representations that are far more invariant, as measured by mutual information estimates. 
On two machine learning benchmarks CelebA and Waterbirds, \methodabbrv~consistently outperforms state-of-the-art approaches that rely on robust optimization, with reductions in subgroup performance gap by up to $10\%$. 
Next, we perform ablations on each stage of our framework: (i) replacing the CycleGAN with state-of-the-art heuristic augmentations worsens the subgroup performance gap by $3.35\%$; (ii) our subgroup consistency regularizer improves robust accuracy by up to $2.5\%$ over prior consistency losses. 
As an extension, we demonstrate that \methodabbrv~can be used in {combination} with heuristic augmentations, providing further gains in robust accuracy of $1.5\%$. 
Lastly, on the challenging real-world skin cancer dataset ISIC, \methodabbrv~improves robust accuracy by $11.7\%$ compared to a group robustness baseline. 

Our results suggest that model patching is a promising direction for improving subgroup robustness in real applications.
Code for reproducing our results  is available at \url{https://github.com/HazyResearch/model-patching}.%

\begin{figure*}
    \centering
    \includegraphics[width=\textwidth]{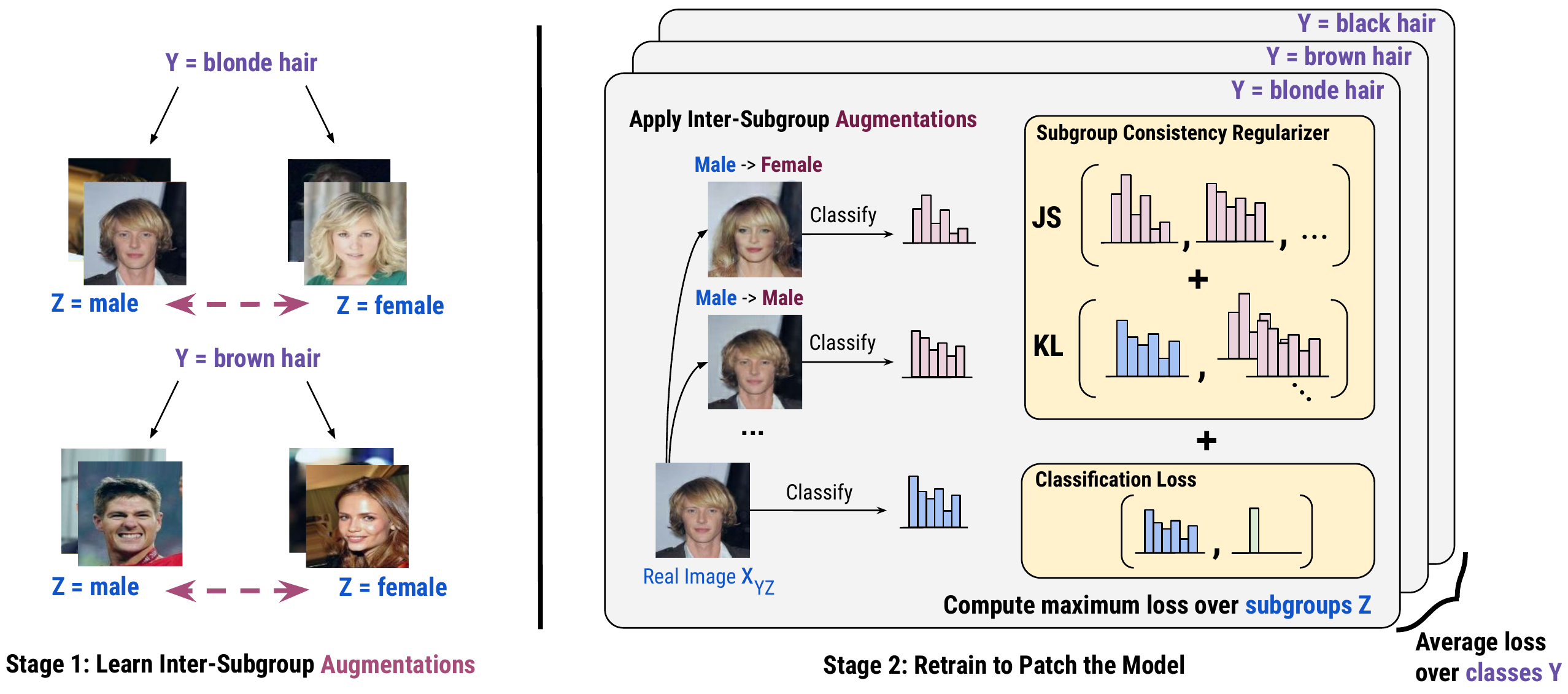}
    \caption{\small {\bf The model patching framework.} (Left) The class-subgroup hierarchy with each class $Y$ divided into subgroups (e.g. $Y={\rm blonde\,\, hair}$ into $Z \in \{ {\rm male}, {\rm female} \}$). We learn inter-subgroup augmentations to transform examples between subgroups of a class. (Right) To patch the classifier, we augment examples by changing their subgroup membership and then train with our subgroup consistency loss and robust objective.} %
    \label{fig:splash}
\end{figure*}

%% file: src/method.tex
\section{\methodabbrv: CycleGAN Augmented Model Patching}
\label{sec:method}

In this section, we walk through \methodabbrv's two-stage framework (Figure~\ref{fig:splash}) in detail. In Section~\ref{sec:cyclegan}, we introduce Stage 1 of model patching, learning class-conditional transformations between subgroups. In Section~\ref{sec:stage2}, Stage 2 uses these transformations as black-box augmentations to train a classifier using our new subgroup robust objective (Section~\ref{sec:subgroup_robustness_objective}) and consistency regularizer (Section~\ref{sec:consistency}).
Section~\ref{sec:invariance-analysis} outlines our theoretical analysis on the invariance guarantees of our method. A glossary for all notation is included in Appendix~\ref{sec:glossary}.

\paragraph{Setup.} We consider a classification problem where $\mathcal{X} \subset \mathbb{R}^n$ is the input space, and $\mathcal{Y}=\{1,2,\dots,C \}$ is a set of labels over $C$ classes.
Each class $y \in \mathcal{Y}$ may be divided into disjoint subgroups $Z_y \subseteq \mathcal{Z}$. %
Jointly, there is a distribution $P$ over examples, class labels, and subgroups labels $(X, Y, Z)$.
Given a dataset $\{(x_i, y_i, z_i)\}_{i=1}^m$, our goal is to learn a class prediction model $f_\theta : \mathcal{X} \to \Delta^C$ parameterized by $\theta$, where $\Delta^C$ denotes a probability distribution over $\mathcal{Y}$. %

\subsection{Stage 1: Learning Inter-Subgroup Transformations}
\label{sec:cyclegan}
The goal of the first stage is to learn transformations $F_{z \to z'} : \mathcal{X}_z \to \mathcal{X}_{z'}$ that translate examples in subgroup $z$ to subgroup $z'$, for every pair of subgroups $z, z' \in Z_y$ in the same class $y$. 

Recent work has made impressive progress on such cross-domain generative models, where examples from one domain are translated to another, ideally preserving shared semantics while only changing domain-specific features.
In this work, we use the popular CycleGAN model \cite{zhu2017unpaired} to learn mappings between pairs of subgroups, although we note that it is possible to substitute other models. Given datasets $\{x_z\}_{i=1}^p$, $\{x_{z'}\}_{i=1}^{p'}$ from a pair of subgroups $z, z' \in Z_y$, we train a CycleGAN $F_{z \to z'}$ to transform between them.
When classes have more than two subgroups, pairwise models can be trained between subgroups, or multi-domain models such as the StarGAN~\cite{choi2018stargan} can be used. 
We include a review of CycleGANs in Appendix~\ref{sec:cg_background}.

Given these transformations $\{F_{z\to z'} \}_{z, z' \in Z_y}$, we generate augmented data for every training example $(x, y, z)$ by passing it through all $F_{z \to z'}, z' \in Z_y$.
We denote these generated examples  %
 $\xt_{Z_y} := \{\xt_{z'}\}_{z' \in Z_y}$ where $\xt_{z'} = F_{z\to z'} (x)$.
For convenience, $k$ denotes the number of subgroups $|Z_y|$.

Prior work that uses data augmentation to improve robustness has mostly relied on heuristic augmentations~\cite{hendrycks2019augmix}, and focused on robustness to out-of-distribution examples~\cite{hendrycks2019augmix} with empirical studies. In contrast, we learn to transform examples rather than specifying augmentations directly, and focus on improving worst-case subgroup robustness. 
We emphasize that while others have used cross-domain generative models for data augmentation, our novelty lies in {targeting invariance to subgroup features} using this style of augmentation. Past work has focused on domain adaptation~\cite{Huang2018AugGANCD}, few-shot learning~\cite{antoniou2017data}, and data scarcity~\cite{bowles2018gan, Ratner2017LearningTC}, but has not attempted to explicitly control the invariance of the classifier using the learned augmentations. As we describe in our theoretical analysis (Section~\ref{sec:invariance-analysis}), our use of cross-domain models is a natural consequence of the class-subgroup setting.

\subsection{Stage 2: Subgroup Robustness with Data Augmentation}
\label{sec:stage2}
The goal of the second stage is to learn a classifier $f_\theta$ on both the original and augmented data from Stage 1, using our subgroup robust objective (Section~\ref{sec:subgroup_robustness_objective}) and consistency regularizer (Section~\ref{sec:consistency}).
Our robustness objective targets worst-case subgroup robustness, while our consistency regularizer forces the learned classifier to be invariant to subgroup features.
Where relevant, we include discussion here on differences to prior work, with an extended related work in Appendix~\ref{sec:related-work-full}.

\subsubsection{A Subgroup Robustness Objective}
\label{sec:subgroup_robustness_objective}

We review two established objectives for training classifiers with their associated metrics and loss functions, and introduce our new objective to target subgroup robustness (cf.\ Table~\ref{tab:losses_and_metrics}).

\begin{table}[t!]
    \centering
    \small
    \caption{\small Comparison of metrics and losses for classifier training. Here $P_z$ and $\hat{P}_z$ are marginal distributions of $(x, y)$ for the subgroup $z$, and $\alpha_\theta(x, y) = \mathbb{I}[(\argmax f_{\theta}(x))=y]$ denotes correct prediction on an example. }
    \resizebox{\linewidth}{!}{%
    \begin{tabular}{@{}lll@{}}
    \toprule
         & Metric of Interest & Loss $\mathcal{L}(\theta)$ \\
         \midrule
       ${\rm ERM}$  & $\mathbb{E}_{P} \alpha_\theta(x, y)$
    & $\mathbb{E}_{ \hat{P}} \ell(f_\theta(x),y)$ \\ %
       ${\rm GDRO}$ & $\min_{z\in \mathcal{Z}} \mathbb{E}_{P_z} \alpha_\theta(x, y)$
    & $\max_{z\in \mathcal{Z}} \mathbb{E}_{\hat{P}_z} \ell(f_\theta(x),y)$  \\
       ${\rm SGDRO}$ & $|\max_{z\in Z_y}\mathbb{E}_{P_{z}} \alpha_\theta(x, y) - \min_{z\in Z_y}\mathbb{E}_{P_z} \alpha_\theta(x, y) |$ & $\mathbb{E}_{y\in Y}\{ \max_{z\in Z_y} \mathbb{E}_{\hat{P_z}}\ell(f_\theta(x),y)\}$ \\
       \bottomrule
    \end{tabular}%
    }   %
    \label{tab:losses_and_metrics}
\end{table}

\paragraph{Prior work: Empirical Risk Minimization (ERM).}
The usual training goal is to maximize the 
{\em aggregate accuracy}, optimized using the empirical risk with respect to a proxy loss function (Table~\ref{tab:losses_and_metrics}, top).

\paragraph{Prior work: Group Robustness (GDRO).}
In our setting, aggregate performance is too coarse a measure of risk, since classes have finer-grained groups of interest.
This can be accounted for by optimizing the \emph{worst-case} performance over these groups.
Letting $P_z$ denote the conditional distribution of examples associated with subgroup $z\in\mathcal{Z}$,
the {\em robust accuracy} can be quantified by measuring the worst-case performance among all groups.
This can be optimized by minimizing the corresponding group robust risk (Table~\ref{tab:losses_and_metrics}, middle right).
A stochastic algorithm for this group distributionally robust optimization (GDRO) objective was recently proposed~\cite{sagawa2019distributionally}.

\paragraph{Class-conditional Subgroup Robustness (SGDRO).} The GDRO objective treats group structure as a flat hierarchy. While this approach accounts for worst-case subgroup performance, it loses the class-subgroup hierarchy of our setting.
Tailored to this setting, we create the SGDRO training objective (Table~\ref{tab:losses_and_metrics}, bottom right) to optimize class-conditional worst-case subgroup robustness, aggregated over all classes (Figure~\ref{fig:splash} right).
To measure subgroup robustness, we define the {subgroup performance gap} (Table~\ref{tab:losses_and_metrics}, bottom left) for a class as the gap between its best and worst performing subgroups.%

\subsubsection{Subgroup Invariance using a Consistency Regularizer}
\label{sec:consistency}
Standard models can learn to rely on spurious subgroup features when making predictions. Subgroup consistency regularization targets this problem by enforcing consistency on subgroup-augmented data, encouraging the classifier to become invariant to subgroup-features. 

Recall that Stage 2 connects to Stage 1 by receiving augmented data $\xt_{Z_y}$, representing ``imagined'' versions of an example $x$ in all other subgroups $z'$ of its class $y$.
We define the \emph{self-consistency} loss $\mathcal{L}_s$ and \emph{translation-consistency} loss $\mathcal{L}_t$ as follows, where $\mt = \frac{1}{k}\sum_z f_{\theta}(\xt_z)$ denotes the average output distribution on the augmented examples.

\begin{minipage}{.45\linewidth}
\begin{equation}
    \mathcal{L}_s(x, \xt_{Z_y}; \theta) = \frac{1}{k} \sum_{z \in Z_y} {\rm KL}\left( f_\theta(\xt_z) \| \mt \right) \label{eq:self-consistency}
\end{equation}
\end{minipage}
\begin{minipage}{.45\linewidth}
\begin{equation}
\mathcal{L}_t(x, \xt_{Z_y}; \theta) = {\rm KL}\left(f_\theta(x) \| \mt \right) \label{eq:translation-consistency}
\end{equation}
\end{minipage}

The self-consistency loss is the more important component, encouraging predictions on augmented examples to be consistent with each other.
As these augmented examples correspond to one ``imagined'' example per subgroup, self-consistency controls dependence on subgroup features.
Translation consistency additionally forces predictions on the original example to be similar to those of the average CycleGAN-translated examples, ignoring potential artifacts that the CycleGANs generate.

We note that consistency losses have been used before, e.g. UDA~\cite{xie2019unsupervised} and AugMix~\cite{hendrycks2019augmix} use different combinations of KL divergences chosen empirically.
Our regularization~\eqref{eq:self-consistency} is tailored to the model patching setting,
where it has a theoretical interpretation relating to subgroup invariance (Section~\ref{sec:invariance-analysis}).
We show empirical improvements over these alternate consistency losses in Section~\ref{subsection:analyzing_lambda}.

\paragraph{Overall Objective.}
The total consistency loss averages over all examples,
\begin{equation}
\mathcal{L}_c(\theta) = \frac{1}{2}\mathbb{E}_{(x,y)\sim P} \left[ \mathcal{L}_s(x, \xt_{Z_y}; \theta) + \mathcal{L}_t(x, \xt_{Z_y}; \theta)\right].
\label{eq:consistency-loss}
\end{equation}
Combining our SGDRO robust objective and the consistency loss with the consistency strength hyper-parameter $\lambda$ yields the final objective,
\begin{equation}
  \label{eq:overall-objective}
    \mathcal{L}_{\text{\methodabbrv}}(\theta) = \mathcal{L}_{\text{SGDRO}}(\theta) + \lambda \mathcal{L}_\text{c}(\theta).
\end{equation}

%% file: src/analysis.tex
\section{An Information Theoretic Analysis of Subgroup Invariance}
\label{sec:invariance-analysis}

We introduce a framework to analyze our end-to-end approach (equation~\eqref{eq:overall-objective}),
showing that it induces subgroup invariances in the model's features.
First, we review a common framework for treating robustness over discrete groups
that aims to create \emph{invariances},
or independences between the learned model's features $\phi(X)$ and groups $Z$.
We then define a new model for the distributional assumptions underlying the subgroup setting,
which allows us to analyze {stronger invariance guarantees}
by minimizing a mutual information (MI) upper bound.
Formal definitions and full proofs are deferred to Appendix~\ref{sec:invariance-analysis-detailed}.

\paragraph{Prior work: Class-conditioned Subgroup Invariance.}
Prior work~\cite{ganin2016domain,li2018deep,long2018conditional} uses adversarial training to induce subgroup invariances of the form $(\phi(X) \perp Z) \mid Y$,
so that within each class, the model's features $\phi(X)$ appear the same across subgroups $Z$.
We call this general approach \emph{class-conditional domain adversarial training} (CDAT).
Although these works are motivated by other theoretical properties, we show that this approach attempts to induce the above invariance by minimizing a variational \emph{lower} bound of the corresponding mutual information.

\begin{restatable}{lmm}{miadversarial}
  \label{lmm:mi-adversarial}
  CDAT minimizes a lower bound on the mutual information $I(\phi(X); Z \mid Y)$.
\end{restatable}

Since the model's features matter only insofar as they affect the output, for the rest of this discussion we assume without loss of generality that $\phi(X) = \hat{Y}$ is simply the model's prediction.

\paragraph{A Natural Distributional Assumption: Subgroup Invariance on Coupled Sets.}
\begin{wrapfigure}{R}{0.23\textwidth}%
\vspace*{-1em}
  \captionsetup{width=\linewidth}
    \centering
    \includegraphics[width=\linewidth]{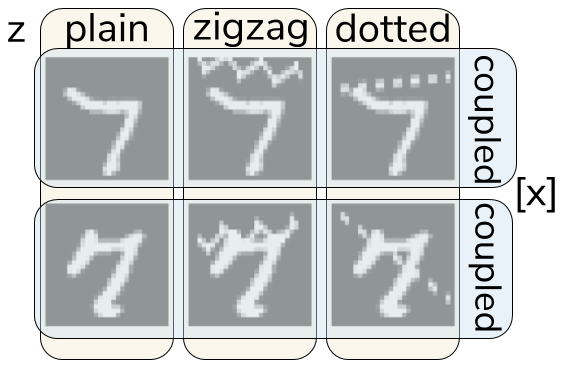}
    \caption{Coupled sets for subgroups of the $Y=7$ class.}%
    \vspace*{-1em}
    \label{fig:coupledset}
\end{wrapfigure}
Although prior work generally has no requirements on how the data $X$ among the groups $Z$ relate to each other,
we note that a common implicit assumption is that there is a ``correspondence'' between examples among different groups.
We codify this distributional assumption explicitly. %

Informally, we say that every example $x$ belongs to a \emph{coupled set} $[x]$, containing one example per subgroup in its ($x$'s) class (Figure~\ref{fig:coupledset}) (\cref{sec:distribution}, \cref{def:coupled-sets}).
$[X]$ is the random variable for coupled sets, i.e. it denotes sampling an example $x$ and looking at its coupled set.
Intuitively, $x' \in [x]$ represent hidden examples in the world that have identical class features to $x$ and differ only in their subgroup features. These hidden examples may not be present in the train distribution and model patching ``hallucinates'' them, allowing models to directly learn relevant class features.

This idea of coupled sets underlies both stages of the framework
and enables stronger invariance guarantees. %
Given this notion, all examples $x$ in a coupled set $[x]$ should have identical predictions in order to be robust across subgroups, modeled by the desired invariance $(\hat{Y} \perp Z) \mid [X]$.
Parallel to Lemma~\ref{lmm:mi-adversarial}, we aim to minimize $I(\hat{Y};Z \mid [X])$.
Note that
$I(\hat{Y};Z \mid [X]) \ge I(\hat{Y};Z \mid Y)$,
which follows from the chain rule for MI (proof in Appendix~\ref{sec:invariance-analysis-detailed}),
so this is a stronger notion of invariance than CDAT permits.
Additionally, the losses from the CycleGAN (Stage 1) and consistency regularizer (Stage 2) combine to form an upper bound on the mutual information rather than a lower bound,
so that optimizing our loss is more appropriate.

\begin{restatable}{thm}{thmmain}%
  \label{thm:main}
  For a model $f_{\theta}$ with outputs $\Yh$, the MI $I(\Yh; Z \mid [X])$ is the Jensen-Shannon Divergence (JSD) of predictions on coupled sets
  $\mathbb{E}_{[x] \sim [X]} JSD\left( f_\theta(x))_{x\in [x]} \right)$.
  In the case of $k=2$ subgroups per class, this can be upper bounded by the CycleGAN and consistency losses
  \begin{align*}
    \E_{(x, y) \sim (X, Y)} \big( \mathcal{L}_{s}(x; \xt_{Z_y}; \theta)^{\frac{1}{2}} + \sum_{z \in Z_y} \mathcal{L}_{CG}^{z}(x; \theta)^{\frac{1}{2}} \big)^2.
  \end{align*}
  In particular, the global optimum of the trained \methodabbrv\ model induces $\Yh \perp Z \mid [X]$.
\end{restatable}

The main idea is that the conditional MI $I(\Yh; Z \mid [X])$ can be related to model's predictions on all elements in a coupled set $[x]$ using properties of the JSD.
However, since we do not have true coupled sets,
the consistency loss~\eqref{eq:consistency-loss} only minimizes a proxy for this JSD using the augmentations $\xt_{Z_y}$.
Using standard GAN results, the divergence between the true and augmented distributions can be bounded by the loss of a discriminator,
and the result follows from metric properties of the JSD.

Thus, the CycleGAN augmentations (Stage 1) and our consistency regularizer (Stage 2) combine to provide an upper bound on our MI objective,
tying together the model patching framework neatly. %

%% file: src/experiments.tex
\section{Experiments}
\label{sec:experiments}

Our goal is to demonstrate that \methodabbrv~can take advantage of the learned subgroup augmentations and consistency regularizer to improve robust 
and aggregate
accuracy, while reducing the subgroup performance gap
(defined in Table~\ref{tab:losses_and_metrics}). We validate \methodabbrv~against both standard training with no subgroup knowledge (ERM) and other baselines aimed at improving group robustness across 4 datasets. We also conduct extensive ablations to isolate the benefit of the learned inter-subgroup transformations over standard augmentation, and the subgroup consistency regularizer over prior consistency losses.

\begin{table*}[t]
\centering
\caption[]{\small A comparison between \methodabbrv\ and other methods on 3 benchmark datasets. Evaluation metrics include robust \& aggregate accuracy and the subgroup performance gap, calculated on the test set. Results are averaged over 3 trials (one standard deviation indicated in parentheses).}
\resizebox{\textwidth}{!}{%
\begin{tabular}{llrrrrrrrr}
\toprule 
 \multirow{ 2}{*}{\bf Dataset}          &  \multirow{ 2}{*}{\bf Method}      & \multicolumn{4}{c}{{\bf Subgroup} $\quad Y$} & \bf{Aggregate} & \bf{Robust} & \multicolumn{2}{c}{ \bf{Subgroup Gap (\%)} }    \\
                      &                   & \multicolumn{4}{c}{{\bf Acc. (\%)} $\quad\hidden{\quad\quad\quad\quad}\, Z$} & \bf{Acc. (\%)} & \bf{Acc. (\%)} & \multicolumn{2}{c}{$Y$} \\
\midrule
& & even & even & odd & odd & & & \multirow{ 2}{*}{even} & \multirow{ 2}{*}{odd} \\ 
& & clean & zigzag & clean & zigzag \\ \cmidrule{3-6}\cmidrule{9-10}
{\bf MNIST-}           & \bf {ERM}            &    86.96 & 73.51 & 71.47 & 75.21 & 76.75 (1.60) & 71.47 (1.50) & 13.45 & 3.73            \\
 
                  {\bf Correlation}       & \bf {IRM}                 &  94.68 & 69.30 & 81.77 & 93.53  & 84.85 (5.42)          & 69.30 (3.29) & 25.38 & 11.76         \\
                      & \bf {CDAT}       &  94.63       &  72.85     & 79.21        &               92.97 & 84.93 (5.84) & 72.85 (3.47) & 21.78 & 13.76                \\
                         & \bf {GDRO}                &  98.10 & 93.31 & 96.82 & 97.15  & 96.35 (0.49)       & 93.31 (1.30)     & 4.79 & 0.79          \\
                      & \bf {\methodabbrv \hidden{~(ours)} }            &     98.85 & 97.89 & 97.98 & 97.87 & {\bf 97.55} (0.46) & {\bf 97.77} (0.42) & {\bf 0.96} & {\bf 0.17}            \\
\midrule
&  & non-blonde & non-blonde & blonde & blonde & & & \multirow{ 2}{*}{non-blonde} & \multirow{ 2}{*}{blonde} \\
&  & female & male & female & male \\ \cmidrule{3-6}\cmidrule{9-10}
{\bf CelebA-}           & \bf {ERM}     &  81.09     & 98.08   &  98.13     & 60.04      & 88.26 (1.88)              & 62.22 (6.83) & 16.99 & 38.09              \\
{\bf Undersampled}                       & \bf{GDRO}        &  89.26 & 92.24  & 94.08  & 82.20  & 90.91 (0.78)           & 82.20 (3.13) & 2.98 & 11.88          \\
                      & \bf {\methodabbrv\hidden{~(ours)}}  &  92.15 & 93.73 & 91.13 & 83.53  & {\bf 92.90} (0.35)    & {\bf 83.90} (1.31)  & {\bf 1.83} & {\bf 8.07}  \\
\midrule
 & & landbird & landbird & waterbird & waterbird & & & \multirow{ 2}{*}{landbird} & \multirow{ 2}{*}{waterbird} \\ 
 & & land & water & land & water \\ \cmidrule{3-6}\cmidrule{9-10}
{\bf Waterbirds}       & \bf {ERM}     &  98.92     & 75.12     & 72.71     & 94.95      & 86.31 (0.39)              & 72.71 (2.36) & 23.80 & 22.24              \\
                      & \bf {GDRO}   &   94.46  & 83.81 & 88.19 & 92.36  & 89.39 (0.19) & 83.81 (0.39) & 10.65 & 4.17         \\
                      & \bf{\methodabbrv\hidden{~(ours)}}  &  90.84 & 90.40 & 89.69  & 89.58  &  {\bf 90.89} (0.87)   & {\bf 89.12} (0.36) & {\bf 0.43} & {\bf 1.04}   \\
\bottomrule
\end{tabular}%
}

\label{tab::results-comparison-with-other-methods}
\end{table*}

\paragraph{Datasets.} We briefly describe the datasets used, with details available in Appendix~\ref{sec:full_dataset_info}.

\vspace{0.2cm}\noindent\textit{MNIST-Correlation.} We mix data from MNIST~\cite{lecun1998gradient} and MNIST-Corrupted~\cite{mu2019mnist} to create a controlled setup for analyzing subgroup performance. Digit parity classes $Y \in \{ {\rm even}, {\rm odd}\}$ are divided into subgroups $Z \in \{{\rm clean}, {\rm zigzag}\}$ from MNIST and MNIST-Corrupted respectively. $Y$ and $Z$ are highly correlated, so that most ${\rm even}\,\, ({\rm odd})$ digits are ${\rm clean}\,\, ({\rm zigzag})$.
 
\vspace{0.2cm}\noindent\textit{CelebA-Undersampled.} Following \cite{sagawa2019distributionally}, we classify hair color $Y \in \{ {\rm non{\text -}blonde}, {\rm blonde}\}$ in the CelebA faces dataset \cite{liu2015deep}. 
Subgroups are based on gender $Z = \{ {\rm female}, {\rm male}\}$. 
We subsample the set of ${\rm non{\text -}blonde\,\,women}$ so that most ${\rm non{\text -}blonde \,\,({\rm blonde})}$ examples are ${\rm men} \,\,({\rm women})$.

\vspace{0.2cm}\noindent\textit{Waterbirds.} In this dataset to analyze spurious correlations \cite{sagawa2019distributionally}, birds $Y \in \{{\rm landbird},{\rm waterbird}\}$ are placed against image backgrounds $Z \in \{ {\rm land},{\rm water}\}$, with ${\rm waterbirds}\,\, ({\rm landbirds})$ more commonly appearing against ${\rm water}\,\, ({\rm land})$.

\vspace{0.2cm}\noindent\textit{ISIC.} ISIC (International Skin Imaging Collaboration) is a skin cancer dataset~\cite{codella2018skin}. 
We classify $Y \in \{{\rm benign}, {\rm malignant}\}$ cancers, with bandages $Z$ appearing on $\sim 50 \%$ of only ${\rm benign}$ images.

\paragraph{Methods.}
\methodabbrv\ instantiates model patching as described in Section~\ref{sec:method}.
We use the original CycleGAN model with default hyperparameters (Appendix~\ref{sec:cyclegan_training_details}).
We compare against {ERM} and {GDRO}~\cite{sagawa2019distributionally} (Table~\ref{tab:losses_and_metrics}), which respectively minimize the standard risk
and robust risk
(over all subgroups) on the training set. 
On MNIST-Correlation, we additionally compare against the {IRM}~\cite{arjovsky2019invariant} and {CDAT}~\cite{li2018deep} baselines which target invariance assumptions (details in Appendix~\ref{sec:baselines}).
All classifiers are fine-tuned using a ResNet-50 architecture, with pretrained ImageNet weights. Detailed information about experimental setups and hyperparameters are provided in Appendix~\ref{sec:experimental_details}.

\subsection{Subgroup Robustness and Invariance on Benchmark Datasets}
\label{sec:benchmarks}
We first compare all methods on the benchmark datasets, with results summarized in Table~\ref{tab::results-comparison-with-other-methods}.%

\paragraph{\methodabbrv~increases aggregate and robust accuracy while closing the subgroup gap.}  %
On all datasets, \methodabbrv~improves both aggregate and robust accuracy by up to ${5.3\%}$,
mitigating the tradeoff that other methods experience. %
 \methodabbrv\ also balances out the performance of subgroups within each class, e.g., on Waterbirds, reducing this subgroup gap by $10.22\%$ on ${\rm landbirds}$ compared to GDRO.

\begin{wraptable}{r}{0.45\linewidth}%
\caption{\small Estimated MI between predictions and subgroups computed on MNIST-Correlation (lower is better).}
\resizebox{\linewidth}{!}{%
\begin{tabular}{@{}llllll@{}}
  \toprule
                        & \textbf{ERM} & \textbf{IRM} & \textbf{CDAT} & \textbf{GDRO} & \textbf{\methodabbrv} \\
                        \midrule
\textbf{MI Estimate} & 0.67         & 0.69         & 0.69          & 0.33          & \textbf{0.02}  \\
\bottomrule
\end{tabular}%
}

\label{tab::mutual_info}
\end{wraptable}

\paragraph{\methodabbrv~learns subgroup-invariant representations.}
To measure the invariance of models, we report an estimate of the mutual information defined in Lemma~\ref{lmm:mi-adversarial},
calculated using class-conditional domain prediction heads (Appendix~\ref{sec:mi_measurement}).
Table~\ref{tab::mutual_info} illustrates that \methodabbrv~is the only method that successfully makes the model invariant to subgroups in the dataset.%

\subsection{Model Patching Ablations}
We perform ablations on the major components of our framework:
(1) substituting learned augmentations with alternatives like heuristic augmentations in Stage 1, and 
(2) substituting prior consistency losses for our subgroup consistency regularizer in Stage 2.

\subsubsection{Effect of Learned Augmentations}
\label{sec:stage_1_ablation}

\begin{table}%
\centering
\caption[]{\small Ablation analysis (Section \ref{sec:stage_1_ablation}) that varies the consistency penalty coefficient $\lambda$. For brevity, we report the maximum subgroup performance gap over all classes.
}
\resizebox{0.5\linewidth}{!}{%
\begin{tabular}{@{}llll@{}}
\toprule
\multirow{3}{*}{\textbf{Method}}                                                           & \multicolumn{3}{c}{\textbf{Robust Acc. (\%)}}                     \\
                                                                                          & \multicolumn{3}{c}{{\textbf{Max Subgroup Gap}}}            \\ \cmidrule(l){2-4} 
                                                                                          & $\lambda = 20$         & $\lambda = 50$         & $\lambda = 200$ \\ \midrule
\multirow{2}{*}{\textbf{\begin{tabular}[c]{@{}l@{}}Subgroup\\ Pairing\end{tabular}}}       & 74.22                  & 71.88                  & 74.22           \\
                                                                                          & {19.53}         & {23.43}         & {23.06}  \\ \midrule
\multirow{2}{*}{\textbf{\begin{tabular}[c]{@{}l@{}}Heuristic\\ Augmentation\end{tabular}}} & 87.50                  & 88.54                  & 79.17           \\
                                                                                          & {6.95}          & {6.48}          & {37.50}  \\ \midrule
\multirow{2}{*}{\textbf{\methodabbrv}}                                                            & 82.03                  & 83.33                  & \textbf{89.06}  \\
                                                                                          & {12.50}         & {10.84}         & {3.13}   \\  \midrule\midrule
\multirow{2}{*}{\textbf{\methodabbrv\ + Heuristic}}                                                & 89.06                  & \textbf{90.62}         & 53.45           \\
                                                                                           &   {{0.21}} & {{1.30}} & {19.39}  \\ \bottomrule
\end{tabular}%
}

\label{tab::results-consistency-loss}
\end{table}

We investigate the interaction between the type of augmentation used and the strength of consistency regularization, by varying the consistency loss coefficient $\lambda$ on Waterbirds (Table~\ref{tab::results-consistency-loss}). 
We compare to: (i) {\it subgroup pairing}, where consistency is directly enforced on subgroup examples from a class without augmentation and (ii) {\it heuristic augmentations}, where the CycleGAN is substituted with a state-of-the-art heuristic augmentation pipeline~\cite{hendrycks2019augmix} (Appendix~\ref{sec:baselines}) containing rotations, flips, cutout etc. 
Our goal is to validate our theoretical analysis, which suggests that strong consistency training should help most when used with the coupled examples generated by the CycleGAN. We expect that the ablations should benefit less from consistency training since, (i) subgroup pairing enforces consistency on examples across subgroups that may not lie in the same coupled set; and (ii) heuristic augmentations may not change subgroup membership at all, and may  even change class membership.

\paragraph{Strong consistency regularization enables \methodabbrv's success.}
As $\lambda$ increases from $20$ to $200$, \methodabbrv's robust accuracy rises by $7\%$ while the subgroup gap is $9.37\%$ lower. 
For both ablations, performance {deteriorates} when $\lambda$ is large.
Subgroup pairing is substantially worse ($14.84\%$ lower) since it does not use any augmentation, and as we expected does not benefit from increasing $\lambda$.
Heuristic augmentations (e.g. rotations, flips) are not targeted at subgroups and can distort class information (e.g. color shifts in AugMix), and we observe that strongly enforcing consistency ($\lambda = 200$) makes these models much worse.
Overall, these results agree with our theoretical analysis.

\paragraph{\methodabbrv\ combines flexibly with other augmentations.} Empirically, we observe that performing heuristic augmentations in addition to the CycleGAN (\methodabbrv{} + Heuristic) can actually be beneficial, with a robust accuracy of $90.62\%$ and a subgroup gap that is $1.83\%$ lower than using \methodabbrv\ alone at their best $\lambda$. %

\subsubsection{Analyzing the Subgroup Consistency Regularizer}
\label{subsection:analyzing_lambda}
Next, we investigate our choice of consistency regularizer, by substituting it for (i) a triplet Jensen-Shannon loss~\cite{hendrycks2019augmix} and (ii) a KL-divergence loss~\cite{xie2019unsupervised} in \methodabbrv\ (Figure~\ref{fig:consistency_loss_ablations}). Our goal is to demonstrate that our theoretically justified regularizer reduces overfitting, and better enforces subgroup invariance.%
 
\paragraph{Consistency regularization reduces overfitting.} Figure~\ref{fig:consistency_loss_ablations} illustrates the train and validation cross-entropy loss curves for \methodabbrv\ and GDRO on the small $({\rm landbird}, {\rm water})$ Waterbirds subgroup (184 examples). Consistency regularization shrinks the gap between train and validation losses, strongly reducing overfitting compared to GDRO.

\paragraph{Alternative consistency losses deteriorate performance.} %
As expected, substituting the subgroup consistency loss with either the triplet-JS loss or the KL loss in \methodabbrv\ reduces robust accuracy significantly ($-2.5\%$ on Waterbirds). Interestingly, our subgroup consistency regularizer improves over prior consistency losses even when used with heuristic augmentations. %

\subsubsection{Additional GAN Ablations}
Several GAN works highlighted in Appendix~\ref{sec:related-work-full} have been used for data augmentation.
However, they have focused on metrics such as image quality and aggregate accuracy, as opposed to robust accuracy.
In Appendix~\ref{sec:gan_baselines}, we consider three other GAN baselines in addition to CycleGAN, either by themselves as a pure augmentation method, or integrated in the model patching pipeline.
Model patching consistently improves the robust performance of each base model.

\subsection{Real-World Application in Skin Cancer Classification}

\begin{figure}
\begin{subfigure}{0.4\linewidth}
    \centering
    \includegraphics[width=\linewidth]{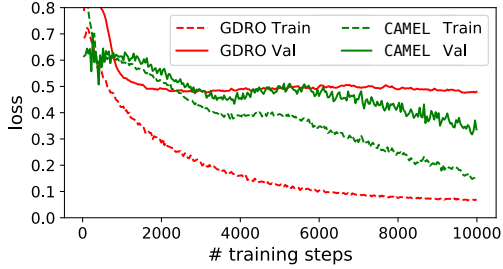}
    \label{fig:train-val}
\end{subfigure}\hfill
\begin{subtable}{0.55\linewidth}
\vspace*{-1.5em}
\centering
\small
\resizebox{\linewidth}{!}{%
\begin{tabular}{lllll}
    \toprule
                                                  & \multicolumn{2}{c}{{\bf Learned Aug.}}  & \multicolumn{2}{c}{{\bf Heuristic Aug.}} \\
    \cmidrule{2-5} %
                                                  & {\bf Triplet JS}                        & {\bf KL}                                  & {\bf Triplet JS}                 & {\bf KL} \\
    \midrule
    \begin{tabular}[c]{@{}l@{}}{\bf Performance} \\ {\bf Change}                           \\ (vs. \methodabbrv{}                     \\ Consistency Loss) \end{tabular} & -2.50     & -0.83 & -2.08 & -1.04 \\
    \bottomrule
\end{tabular}
}
\label{tab:alternative-consistency-loss}
\end{subtable}

\caption{\small \textbf{Consistency loss ablations on Waterbirds}. (Left) loss curves on the $({\rm landbird}, \,{\rm water})$ subgroup. The addition of the \methodabbrv\ consistency loss to GDRO reduces overfitting. (Right) Robust accuracy decrease with alternate consistency losses (Triplet JS~\cite{hendrycks2019augmix} and KL~\cite{xie2019unsupervised}) on \methodabbrv-generated data or heuristic augmentations.}
\label{fig:consistency_loss_ablations}
\end{figure}

We conclude by demonstrating that \methodabbrv\ can improve performance substantially on the real-world ISIC~\cite{codella2018skin} skin cancer dataset (Table~\ref{tab:isic}). We augment only the benign class, which is split into subgroups due to the presence of a colored bandage (Figure~\ref{fig:isic}) while the malignant class contains no subgroups. We also additionally report AUROC, as is conventional in medical applications. %

\begin{wraptable}{r}{0.4\linewidth}
\centering
\caption{\small Comparison on ISIC.}
\resizebox{0.9\linewidth}{!}{%
\begin{tabular}{llll}
\toprule
\multirow{2}{*}{\bf Method} &\multicolumn{2}{c}{ Evaluation Metric}   \\
& \textbf{Robust Acc.} & \textbf{AUROC}\\
 \midrule
  ERM & 65.59 (1.17) & {\bf 92.48} (0.80) \\   
  GDRO & 64.97 (3.15) & 89.50 (2.50) \\   
  \methodabbrv & {\bf 77.45} (0.35)  & {\bf 92.47} (0.38) \\ 
 \bottomrule
\end{tabular}%
}
\label{tab:isic}
\end{wraptable}
 \methodabbrv\ substantially improves robust accuracy by $11.7\%$ and importantly, increases accuracy on the critical malignant cancer class from $65.59\%$ (ERM) and $64.97\%$ (GDRO) to $78.86\%$ (Appendix~\ref{sec:isic_spurious}). %
While standard ERM models spuriously correlate the presence of the colored bandage with the benign class, \methodabbrv~reduces the model's dependence on spurious features. We verify this by constructing a modified ISIC subgroup (Appendix~\ref{sec:isic_spurious}) for the malignant class that also contains bandages.
Figure~\ref{fig:isic} illustrates using GradCAM~\cite{selvaraju2017grad} that \methodabbrv\ removes the model's reliance on the spurious bandage feature, shifting attention to the skin lesion instead.%

%% file: src/conclusion.tex
\section{Conclusion}
Domain experts face a common problem: how can classifiers that exhibit unequal performance on different subgroups of data be fixed? To address this, we introduced model patching, a new framework that improves a classifier's subgroup robustness by encouraging subgroup-feature invariance. 
Theoretical analysis and empirical validation suggest that model patching can be a useful tool for domain experts in the future. 

\clearpage

%% file: src/glossary.tex
\section{Glossary of Notation}
\label{sec:glossary}

We provide a glossary of notation used throughout the paper.

\begin{table}[ht]

\centering
    \caption{\small Summary of notation used throughout this work.}
    \resizebox{\textwidth}{!}{%
    \begin{tabular}{@{}lll@{}}
      \toprule
      & Notation                                            & Description                                                                                           \\
      \midrule
      \textbf{Preliminaries} & $x, y, z$                                           & Example, class, subgroup                                                                              \\
      & $X, Y, Z$                                           & Random variables for examples, classes, and subgroups                                                 \\
      & $P$                                                 & The joint distribution for $X, Y, Z$                                                                  \\
      & $P_y, P_z$                                          & The distribution for $X$ conditioned on class $y$ or subgroup $z$                                     \\
      & $\mathcal{X}, \mathcal{Y}, \mathcal{Z}$             & Domains for $X, Y, Z$                                                                                 \\
      & $Z_y \subset \mathcal{Z}$                           & The subgroups belonging to class $y$                                                                  \\
      & $Y_z \in \mathcal{Z}$                                & The class of a subgroup $z$                                                                           \\
      & $f_\theta : \mathcal{X} \to \Delta^{|\mathcal{Y}|}$ & The parameterized class prediction model, returning a categorical distribution over $\mathcal{Y}$     \\
      & $\hat{Y}$                                           & A random variable with support $\mathcal{Y}$ indicating a random sample from the output of $f_\theta$ \\
      \midrule
      \textbf{Coupled sets} & $[x]$                                               & A coupled set                                                                                                     \\
      \textbf{and augmentations} & $[X]$                                               & Random variable for coupled sets                                            \\
      & $[x]_z$                                             & Example belonging to subgroup $z$ in the coupled set $[x]$                                            \\
      & $x_{Z_y}$                                           & The coupled set (Definition~\ref{def:coupled-sets}) of examples in $x$'s class $y$. Same as $[x]$. \\
      & $[\tilde{x}]$                                       & An augmented coupled set                                                                               \\
      & $[\tilde{x}]_z, [x]_{\tilde{z}}$                    & Example belonging to subgroup $z$ in the augmented coupled set $[\tilde{x}]$                           \\
    
      & $\tilde{x}_{Z_y}$                                   &  The augmented coupled set of examples in $\tilde{x}$'s class $y$. Same as $[\tilde{x}]$.                                                                                                     \\
      & $k$                                                 & Number of subgroups in any (generic) class   \\
      \midrule
      \textbf{Model components} & $\mathcal{L}_{CG}$        & Sum of CycleGAN consistency and identity losses                                                                                                      \\
      \textbf{and losses} & $\mathcal{L}_{s}$                                   & Self-consistency loss (Eq~\ref{eq:self-consistency})                                                                                                       \\
      & $\mathcal{L}_{t}$                                   & Translation-consistency loss (Eq~\ref{eq:translation-consistency})                                                                                                       \\
      & $\mathcal{L}_{c}$                                   & Total consistency loss (Eq~\ref{eq:consistency-loss})                                                                                                      \\
      & $L : \mathcal{X}^2 \to \R$                          & A distance function, used for CycleGAN consistency losses                                             \\
      & $\lambda$                                           & Hyperparameter controlling the strength of the consistency loss                                                                                                       \\
      & $KL(\cdot)$                                         & The KL divergence                                                                                     \\
      & $JS(\cdot)$                                         & The Jensen-Shannon divergence (Definition~\ref{def:js})                                               \\
      & $I(\cdot)$                                          & The Mutual Information                                                                                \\
      \bottomrule
    \end{tabular}%
    }
    \label{tab:glossary}
\end{table}

%% file: src/relatedwork_full.tex
\section{Extended Related Work}
\label{sec:related-work-full}

We provide a comprehensive overview of related work and highlight connections to our work below.

\subsection{Overview of Data Augmentation}
Data augmentation is widely used for improving the {aggregate} performance of machine learning models in computer vision~\cite{Krizhevsky2012ImageNetCW,Szegedy2014GoingDW}, natural language processing~\cite{Kolomiyets2011ModelPortabilityEF,Zhang2015CharacterlevelCN,Sennrich2015ImprovingNM} and audio~\cite{Ko2015AudioAF,Cui2015DataAF}. The theoretical motivation for data augmentation is largely based on the {tangent propagation} formalism~\cite{Simard1991TangentP,Simard1992EfficientPR,Simard1998TransformationII,Dao2018AKT} which expresses the desired invariances induced by a data augmentation as tangent constraints on the directional derivatives of the learned model. %

 Early work considered augmentations as image defects~\cite{baird1992document} or stroke warping~\cite{Yaeger1996EffectiveTO} for character recognition. Since then, augmentation is considered an essential ingredient in computer vision~\cite{lecun1998gradient,Simard2003BestPF}, with commonly used augmentations including random
flips, rotations and crops~\cite{Krizhevsky2012ImageNetCW,Szegedy2014GoingDW,he2016identity}. Applications of augmentation in computer vision include object detection~\cite{dwibedi2017cut,Zoph2019LearningDA} and scene understanding~\cite{Dvornik2018OnTI}%

In natural language processing, common data augmentation techniques include back-translation~\cite{Sennrich2015ImprovingNM,Yu2018QANetCL}, synonym or word substitution~\cite{Kolomiyets2011ModelPortabilityEF,Zhang2015CharacterlevelCN,Wang2015ThatsSA,Fadaee2017DataAF,Kobayashi2018ContextualAD}, noising~\cite{Xie2017DataNA}, grammar induction~\cite{Jia2016DataRF}, text editing~\cite{Wei2019EDAED} and other heuristics~\cite{Deschacht2009SemisupervisedSR,Silfverberg2017DataAF}. In speech and audio applications, augmentation is also commonly used, through techniques such as vocal tract length warping~\cite{Jaitly2013VocalTL,Ko2015AudioAF} and stochastic feature mapping~\cite{Stylianou1998ContinuousPT,Cui2015DataAF}.

In this work, we perform an empirical evaluation on image classification tasks although our ideas can be extended to classification of other modalities such as speech and text.

\subsection{Augmentation Primitives and Pipelines}
Next, we highlight the particular augmentation primitives that have been used in prior work. Our work is differentiated by the use of {learned} augmentation primitives using CycleGANs~\cite{zhu2017unpaired}, as well as a theoretical justification for this choice.%

\paragraph{Hand-Crafted Augmentation Primitives.} Commonly used primitives are typically heuristic transformations, such as rotations, flips or crops~\cite{Krizhevsky2012ImageNetCW,Szegedy2014GoingDW}. Recent work has hand-crafted more sophisticated primitives, such as Cutout~\cite{devries2017improved}, Mixup~\cite{zhang2017mixup}, CutMix~\cite{yun2019cutmix} and MixMatch~\cite{berthelot2019mixmatch}. While these primitives have culminated in compelling performance gains~\cite{cubuk2019autoaugment,cubuk2019randaugment}, they produce unnatural images and distort image semantics.

\paragraph{Assembling Augmentation Pipelines.} Recent work has explored learning {augmentation policies} -- the right subset of augmentation primitives, and the order in which they should be applied. The learning algorithm used can be reinforcement learning~\cite{ratner2017learning, cubuk2019autoaugment} or random sampling~\cite{cubuk2019randaugment}. More computationally efficient algorithms for learning augmentation policies have also been proposed~\cite{ho2019population,lim2019fast}. 

These pipelines are primarily derived from the fixed set of generic image transformations we discussed earlier, and do not directly target specific attributes. By contrast, we consider {learning} augmentation primitives that target subgroup robustness, and additionally demonstrate in Section~\ref{subsection:analyzing_lambda} that heuristic augmentations can complement \methodabbrv~to yield additional performance gains.

\paragraph{\bf Learned Augmentation Primitives.} There is substantial prior work in learning image transformations that produce {semantic}, rather than superficial changes to an image. A common paradigm is to learn a semantically meaningful data representation, and manipulate embeddings in this representation to produce a desired transformation. Transformations can then be expressed as vector operations over embeddings~\cite{Reed2015DeepVA,upchurch2017deep} or manifold traversals~\cite{Reed2014LearningTD,Gardner2015DeepMT}. Alternative approaches rely on training conditional generative models~\cite{Brock2016NeuralPE,isola2017image,zhu2017unpaired,almahairi2018augmented,choi2018stargan} that learn a mapping between two or more image distributions. Much of this prior work is motivated by the need for sophisticated tools for image editing~\cite{upchurch2017deep,Karras2018ASG} \emph{e.g.} for creative applications of machine learning~\cite{Mazzone2019ArtCA}. 

Closer to our setting is work that explores the use of these transformations for data augmentation. A prominent use case focuses on imbalanced datasets, where learned augmentations are used to generate examples for underrepresented classes or domains. Examples include BaGAN~\cite{mariani2018bagan}, DAGAN~\cite{antoniou2017data}, TransferringGAN~\cite{wang2018transferring} and others~\cite{Tran2017ABD,Molano2018GenerativeMF,Zhang2018DADADA,Mounsaveng2019AdversarialLO,Hu2019LearningDM,Beery2019SyntheticEI}. Applications to medical data~\cite{Sandfort2019DataAU,Pesteie2019AdaptiveAO} and person re-identification~\cite{Sun2019UnlabeledSG} have also been explored. 

Our model patching framework differs substantially from these papers, since we focus on \emph{robustness}. We discuss this intersection next. 

\subsection{Data Augmentation and Model Robustness}
Prior work on model robustness has mostly focused on learning models that are robust to bounded $\ell_p$-norm perturbations~\cite{Szegedy2013IntriguingPO,Goodfellow2014ExplainingAH,Papernot2015DistillationAA,MoosaviDezfooli2018RobustnessVC} using ideas such as adversarial training~\cite{Madry2017TowardsDL}. A separate line of work considers consistency training~\cite{zheng2016improving,kannan2018adversarial,hendrycks2019augmix}, where predictions are made invariant to input perturbations, often by minimizing a divergence between the predictions for the original and perturbed examples. Consistency regularization has also been shown to be effective for semi-supervised learning~\cite{xie2019unsupervised}.  

\paragraph{\bf Consistency training.}
We contrast equation~\eqref{eq:consistency-loss} with consistency losses from prior work.
Unsupervised Data Augmentation (UDA)~\cite{xie2019unsupervised} simply controls an asymmetric divergence between the original example and each augmented example individually $\sum_z {\rm KL}(f(x) \| f(\xt_z))$.
AugMix~\cite{hendrycks2019augmix} uses a Jensen-Shannon divergence
\[
\frac{1}{k+1} \left[{\rm KL}\left(f(x) \| \mt \right) + \sum_{z \in Z_y} {\rm KL}\left( f(\xt_z) \| \mt \right) \right]
\]
where $\mt = \frac{1}{k+1}\left[f(x) + \sum_i f(\xt_i) \right]$.
This can be seen as a version of our consistency, but with different weights and a different mean distribution that the KL's are being computed against. Our loss \eqref{eq:consistency-loss} has an important asymmetry between the original example $x$ and the augmentations $\xt_i$.
One reason to prefer it is simply noting that as the number $k$ of subgroups grows, the AugMix loss tends to the second term,
and does not control for the discrepancy between predictions on the original domain $f(x)$ and the augmented ones $f(\tilde x_i)$. 
Our consistency regularization instead allows us to bound a mutual information objective between variables in the joint subgroup distribution, yielding a tractable and interpretable objective (Section~\ref{sec:invariance-analysis}).
In addition, we compare with these consistency losses and provide empirical results in Section~\ref{subsection:analyzing_lambda}. 

Robustness to more general augmentations has also been explored~\cite{Odena2016ConditionalIS,Engstrom2017ARA,Kanbak2017GeometricRO,Baluja2017AdversarialTN,Song2018ConstructingUA,Xiao2018GeneratingAE,Qiu2019SemanticAdvGA}, but there is limited work on making models more robust to semantic data augmentations. The only work we are aware of is  AdvMix~\cite{gowal2019achieving}, which combines a disentangled generative model with adversarial training to improve robustness. 

Our work contributes to this area by introducing the model patching framework to improve robustness in a targeted fashion. Specifically, under the data-generating model that we introduce, augmentation with a CycleGAN~\cite{zhu2017unpaired} model allows us to learn predictors that are invariant to subgroup identity.

\subsection{Learning Robust Predictors}
Recent work~\cite{sagawa2019distributionally} introduced GDRO, a distributionally robust optimization method to improve worst-case accuracy among a set of pre-defined subgroups. However, optimizing the GDRO objective does not necessarily prevent a model from learning subgroup-specific features.
Instead, strong modeling assumptions on the learned features may be required,
\emph{e.g.} Invariant Risk Minimization~\cite{arjovsky2019invariant} attempts to learn an invariant predictor through a different regularization term. However, these assumptions are only appropriate for specialized setups where extreme out-of-domain generalization is desired. Unfortunately, these approaches still suffer from standard learning and generalization issues stemming from a small number of examples in the underperforming subgroup(s) -- even with perfect subgroup information. Additionally, they necessarily trade off average (aggregate) accuracy against a different robust metric.

%% file: src/analysis_full.tex
\section{Detailed Analysis}
\label{sec:invariance-analysis-detailed}

We begin with background material on the CycleGAN (\cref{sec:cg_background}) and the Jensen-Shannon Divergence (\cref{sec:js}).
\cref{sec:distribution} contains a longer discussion of the modeling assumptions in \cref{sec:invariance-analysis},
fleshing out the distributional assumptions and definition of coupled sets.
\cref{sec:mi-adversarial} and \cref{sec:mi-consistency} completes the proofs of the results in \cref{sec:invariance-analysis}.

\input{src/analysis_background}

\subsection{Subgroup Invariance using Coupled Distributions}
\label{sec:distribution}

A common framework for treating robustness over discrete groups
aims to create \emph{invariances},
or independencies between the learned model's features and these groups.
We review this approach,
before defining a new model for the distributional assumptions used in this work.
The notion of coupled sets we introduce underlies both stages of the framework
and allows for stronger invariance guarantees than previous approaches,
which will be analyzed in \cref{sec:mi-consistency}.

\paragraph{\bf Class-conditioned Subgroup Invariance.}
In order for a model to have the same performance over all values of $Z$,
intuitively it should learn ``$Z$-invariant features'',
which can be accomplished in a few ways.
Invariant Risk Minimization (IRM)~\cite{arjovsky2019invariant} calls the $Z$ labels \emph{environments} and aims to induce
$(Y \mid \phi(X)) \perp Z$,
where $\phi(X)$ are the model's features,
so that the classifier does not depend on the environment.
Another line of work treats $Z$ as \emph{domains} and uses adversarial training to induce invariances of the form $(\phi(X) \perp Z) \mid Y$~\cite{ganin2016domain,li2018deep,long2018conditional},
so that within each class, the model's features look the same across domains.
We call this general approach \emph{class-conditional domain adversarial training} (CDAT),
which attaches a domain $Z$ prediction head per class $Y$,
and adopts an adversarial minmax objective so that the featurizer $\phi(X)$ erases $Z$ related information and reduces the model's dependence on $Z$.

\paragraph{\bf Coupling-conditioned Subgroup Invariance.}
Although previous works generally make no assumptions on how the data $X$ among the groups $Z$ relate to each other,
we note that a common implicit requirement is that there is a ``correspondence'' between examples among different groups.
We codify this distributional assumption explicitly with a notion of coupling, which allows us to define and analyze stronger invariances.

In particular,
we assume that the underlying subgroups are paired or coupled,
so that every example can be translated into the other subgroups.
\cref{def:coupled-sets} formalizes our distributional notion of \emph{coupled sets}.

\begin{definition}
  \label{def:coupled-sets}
  For a given distribution $P$, a \emph{coupled set} within class $y$ is a set $\{x_z\}_{z \in Z_y}$ consisting of one example from each subgroup of y, where each example has the same probability.\footnote{Note that this will typically not hold for the training distribution, since some subgroups may be underrepresented, making it much less probable that examples from those subgroups are sampled in a coupled set. However, we are concerned with robustness to a test distribution where the subgroups are of equal importance and equally likely.}
A \emph{coupling} for a distribution $P$ on $(X, Y, Z)$ is a partition of all examples in $\mathcal{X}$ into coupled sets.
For any example $x \in \mathcal{X}$, let $[x]$ denote its coupled set.
Let $[x]_1, \dots, [x]_k$ denote the elements of a coupled set $[x]$ in a class with $k$ subgroups.
Let $[X]$ denote the random variable that samples a coupled set; i.e.\ taking $[x]$ for a random $x$ sampled from any fixed subgroup $z$.
\end{definition}
Additionally, we say that a distribution is \emph{subgroup-coupled} if it satisfies \cref{def:coupled-sets},
i.e.\ it has a coupling.

In the context of subgroups of a class $y$, this assumption entails that every example can be factored into its subgroup and coupled set membership. All examples that are members of a particular coupled set can be thought of as sharing a set of common features that signal membership in the class. Separately, examples that are members of a particular subgroup can be thought to share common features that signal subgroup membership.
Together, these two pieces of information identify any example from class $c$.

We represent this assumption by letting the (unobserved) random variable $[X]$ represent the ``class identity'' of an example $X$, which can be thought of as the class features that aren't specific to any subgroup.
Thus, the full generating process of the data distribution $(X, Y, Z, [X])$ consists of independently choosing a coupled set $[X]$ and subgroup $Z$ within a class $Y$, which together control the actual example $X$.
Note that $[X]$ and $Z$ are both more fine-grained and thus carry more information than $Y$.
This process is illustrated in Figure~\ref{fig:graphical-model}.
Figure~\ref{fig:graphical-model-example} illustrates this concept for the MNIST-Corrupted dataset~\cite{mu2019mnist}. Given a digit class such as $Y=3$, subgroups correspond to corruptions such as zigzags and dotted lines applied to the digits. A coupled set consists of these corruptions applied to a clean digit.

Definition~\ref{def:coupled-sets} allows us to reason about the following stronger invariances.
Given class $y \in \mathcal{Y}$, every example in subgroup $z \in {Z}_y$ implicitly has corresponding examples in all subgroups $Z_{y}$ within its class,
and the learned features for each of these coupled sets should be identical in order to equalize performance between subgroups.
Thus instead of the weaker goal $(\phi(X) \perp Z) \mid Y$,
we use the stronger coupling-conditioned invariance $(\phi(X) \perp Z) \mid Y, [X] = (\phi(X) \perp Z) \mid [X]$.

\begin{figure*}[ht]%
\centering
 \begin{subfigure}{0.45\linewidth}
    \centering
    \includegraphics[width=0.9\linewidth]{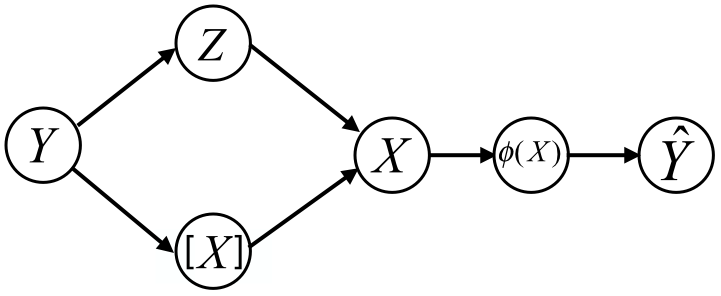}
    \caption{Joint distribution of examples $X$ with their class labels $Y$, subgroup labels $Z$, and coupled sets $[X]$.}
    \label{fig:graphical-model}
  \end{subfigure}
  \hfill
  \begin{subfigure}{0.45\linewidth}
    \centering
    \includegraphics[width=0.9\linewidth]{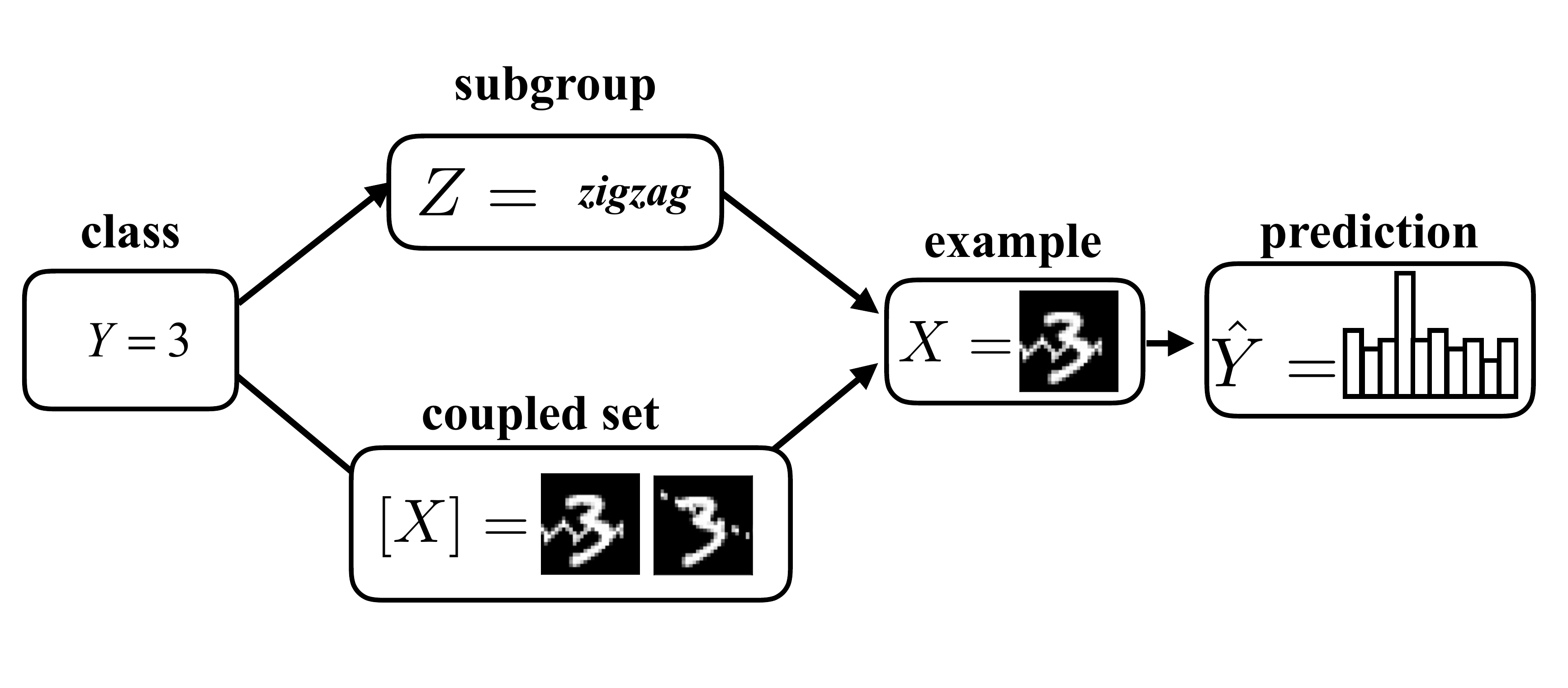}
    \caption{Illustration with the MNIST-Corrupted dataset~\cite{mu2019mnist}, where subgroups $Z$ are different types of corruptions.}
    \label{fig:graphical-model-example}
  \end{subfigure}
  \caption{Subgroup-coupled distributions separate the coupled set to which an example belongs (with respect to their class), from its subgroup label.}
\end{figure*}

Note that since features matter insofar as their effect on the final output $\Yh$, it suffices to look at the case $\phi(X) = \Yh$.
We first show in Section~\ref{sec:mi-adversarial} that CDAT methods target the invariance $(\Yh \perp Z) \mid Y$ by minimizing a lower bound for the conditional mutual information,
$I(\Yh; Z \mid Y)$ (Lemma~\ref{lmm:mi-adversarial}).

In Section~\ref{sec:mi-consistency}, we prove our main result: our combined objective function~\eqref{eq:overall-objective} targets the stronger invariance
$(\Yh \perp Z) \mid [X]$
by upper bounding the corresponding MI,
which can be interpreted as forcing matching outputs for the examples in every coupled set.

\subsection{MI Bounds for Class-conditioned Invariance}
\label{sec:mi-adversarial}

Recall that the high-level goal of CDAT
is to induce independencies between subgroup information and the model's feature representation.
In order to induce the desired invariance $(\phi(X) \perp Z) \mid Y$ of class features from subgroup identities,
a natural approach is to minimize the conditional mutual information $I(\phi(X); Z \mid Y)$,
which is minimized at $0$ when the invariance is satisfied and grows when $\phi(X)$ and $Z$ are predictive of each other.
This mutual information can be estimated using standard techniques.
\begin{lemma}%
  \label{lmm:mi-adversarial}
  CDAT minimizes a lower bound on the mutual information $I(\phi(X); Z \mid Y)$, where $\phi(X)$ is the feature layer where the domain prediction head is attached.
\end{lemma}
\begin{proof}
We have
\begin{align*}%
  I(\phi(X); Z \mid Y) &= H(Z \mid Y) - H(Z \mid \phi(X) , Y)
  \\&= H(Z \mid Y) + \mathbb{E}_{x,y \sim p(x,y)}\mathbb{E}_{z \sim p(z | \phi(x), y)} \left[ \log(p(z | \phi(x), y)) \right]
  \\&\ge H(Z \mid Y) + \mathbb{E}_{x,y \sim p(x,y)}\mathbb{E}_{z \sim p(z | \phi(x), y)} \left[ \log(p_\psi(z | \phi(x), y)) \right]
  \\&= H(Z \mid Y) + \mathbb{E}_{y,z,\phi(x)} \left[ \log(p_\psi(z | \phi(x), y)) \right],
\end{align*}
which bounds the MI variationally through a parametrized conditional model $p_\psi$.
Up to an additive term $H(Z \mid Y)$ which is a constant of the data distribution, this is simply the cross-entropy loss of a model trained on top of the featurizer $\phi$ to predict $Z$ from $\phi(X)$ and $Y$,
which coincides with the domain adversarial training approach.
\end{proof}

By specializing $\phi(X)$ to $\hat{Y}$, we obtain
\begin{corollary}
\label{cor:mi-adversarial}
If CDAT attaches a domain prediction head to the prediction layer $\hat{Y}$, it optimizes a lower bound on $I(\hat{Y}; Z \mid Y)$.
\end{corollary}

Thus, although approaches involving domain adversarial training~\cite{ganin2016domain,li2018deep} motivate their approach through alternate concepts such as $\mathcal{H}$-divergences and GAN-based adversarial games,
we see that they are implicitly minimizing a simple variational estimate for mutual information.

In \cref{sec:experiments}, \cref{tab::mutual_info}'s reported estimate of the mutual information uses \cref{cor:mi-adversarial}.

\subsection{MI Bounds for Coupling-conditioned Invariance}
\label{sec:mi-consistency}

The stronger distributional assumptions of Definition~\ref{def:coupled-sets} allow us to analyze the invariance 
$\phi(X)\perp Z \mid [X]$,
which can be interpreted as forcing matching features for the data in every coupled set.

\paragraph{\bf True Coupled Sets.} Given a subgroup-coupled distribution, access to coupled sets
allows analysis of stronger invariance assumptions.

First, we confirm that this is indeed a stronger notion of invariance, that is
\begin{equation}%
  \label{eq:conditional-mi}    
  I(Z; \phi(X) \mid [X]) \ge I(Z; \phi(X) \mid Y).
\end{equation}
This follows from the chain rule for mutual inequality:
  \begin{equation}%
    \begin{aligned}%
      & I(Z; \phi(X) \mid [X]) = I(Z; \phi(X) \mid Y, [X])
      \\&= I(Z; [X] \mid Y) + I(Z; \phi(X) \mid Y, [X])
      \\&= I(Z; [X], \phi(X) \mid Y)
      \\&= I(Z; \phi(X) \mid Y) + I(Z; [X] \mid Y, \phi(X)).
    \end{aligned}
  \end{equation}
Here, the first two equalities follow from Definition~\ref{def:coupled-sets}
(in particular, $[X]$ and $Z$ are more fine-grained than $Y$),
and the last two follow from the chain rule for mutual information.

In particular, equation~\eqref{eq:conditional-mi} quantifies the intuition that conditioning on an example's coupled set reveals more information then just conditioning on its class.
Conversely, minimizing the LHS of \eqref{eq:conditional-mi} necessarily minimizes the objective $I(Z; \phi(X) \mid Y)$ in ~\cite{li2018deep}, and an additional non-negative term
$I(Z; [X] \mid \phi(X), Y)$ 
relating the features and identity of examples.

Moreover, the features $\phi(X)$ are only relevant insofar as their ability to predict the label.
Specializing $\phi(X)$, this stronger conditional MI is related to the model's predictions;
it is exactly equal to the self-consistency regularizer~\eqref{eq:self-consistency}
if the model had access to true coupled sets $[x]$.

Thus, in the case where $\phi(X) = \hat{Y}$ is simply the model's prediction, this MI is simply the Jensen-Shannon divergence of the model's predictions.
\begin{restatable}{lmm}{lmmconditionalmi}
  \label{lmm:conditional-mi}
\begin{equation}
  \label{eq:conditional-mi2}
  \begin{aligned}
    I(Z; \hat{Y} \mid [X]) &= 
    \mathbb{E}_{[x] \sim [X]} JS\left( f_\theta([x]_1), \dots, f_\theta([x]_k) \right)
    \end{aligned}
\end{equation}
\end{restatable}

\begin{proof}
  For any features $\phi$, the mutual information can be written
  \begin{align*}%
    I(Z; \phi(X) \mid [X]) &= \mathbb{E}_{[X]} I\left(\E[Z \mid [X]];\E[\phi(X) \mid [X]] \right) 
    \\&= \mathbb{E}_{[X]} I\left( Z; \E[\phi(X) \mid [X]] \right)
  \end{align*}
  where the random variable $\E[\phi(X) \mid [X]]$ denotes the formal conditional expectation.
  The second equality follows since $(Z \perp [X]) \mid Y$.

  Consider specializing this to the case when $\phi(X) = \hat{Y}$,
  i.e.\ it represents the random variable where an output class prediction $\hat{Y}$ is sampled from the final class probability predictions $f_\theta(X)$ of the model.
  Since this is distributed as $P_{\hat{Y} \mid X_z} = f_\theta(X_z)$,
  we obtain
  \begin{equation}
    \label{eq:conditional-mi-js}
    \begin{aligned}%
      I(Z; \hat{Y} \mid [X]) &= \mathbb{E}_{[x] \sim [X]}\left[ I\left( Z; \frac{1}{k} \sum_{i \in [k]} f_\theta([x]_i) \right) \right]
      \\&= \mathbb{E}_{[x] \sim [X]} JS\left( f_\theta([x]_1), \dots, f_\theta([x]_k) \right)
    \end{aligned}
  \end{equation}
  where the second equality follows by Proposition~\ref{prop:js-mi}.
\end{proof}

\paragraph{\bf Augmented Coupled Sets.} In practice, we may not have true coupled sets $[x]$.
Instead, we use a generative model such as a CycleGAN as a proxy that provides noisy versions of the coupled set,
denoted $[\tilde{x}] = ([\tilde{x}]_1, \dots, [\tilde{x}]_k)$ where $[\tilde{x}]_i$ are individual augmented examples per subgroup. %
However, the generative augmentation model may not perfectly model the subgroup distribution; for example, it may introduce artifacts.

We can model this distributional assumption explicitly:
\begin{defn}
  \label{def:noisy-examples}
  Each subgroup $z$, which has a distribution $P_z$ over $\mathcal{X}$,
  has a corresponding \emph{augmented subgroup} $\tilde{z}$ with distribution $P_{\tilde{z}}$ representing augmented examples through the generative model(s).
  In particular, we suppose for any coupled set $[x]$, it has realizations $[x]_z$ in subgroup $z$ and $[\tilde{x}]_{z}$ in subgroup $\tilde{z}$.%
\end{defn}
We also use the notation $[\tilde{x}]$ for a generated coupled set and $[\tilde{x}]_z$ as its realization in subgroup $z$ (a specific augmented example). %
Note that $[\tilde{x}]$ and the notation $\tilde{x}_{Z_y}$ from \cref{sec:stage2} refer to the same thing, the set of augmented examples.

Figure~\ref{fig:cyclegan} also illustrates the concept of \cref{def:noisy-examples}:
original domains $A, B$ have corresponding domains $\tilde{A}, \tilde{B}$ that are the images of the generators $F, G$.

We can control the difference between augmented and true subgroup distribution in two ways.
First, the translation-loss $\mathcal{L}_t$~\eqref{eq:translation-consistency} regularizes the average predictions from the augmentations to match those of the original example, constraining the prediction model to ignore general distribution shifts introduced by the generative models.

Moreover, the discrepancy between the loss we are minimizing via CycleGAN-augmented examples
$\mathcal{L}_s = \E_{x}$ $JS\left( f_\theta([\xt]_1), \dots, f_\theta([\xt]_k) \right)$~\eqref{eq:self-consistency} and the true objective $JS\left( f_\theta([x]_1), \dots, f_\theta([x]_k) \right)$ can be bounded by the loss of the pair-conditioned CycleGAN discriminators (Section~\ref{sec:cyclegan}), via metric properties of the JSD.

Models such as CycleGAN directly control the deviation of augmentions from the original examples, via the GAN discriminators and consistency losses.
The following Lemma says that CycleGAN discriminator loss is the divergence between the original distribution in subgroup $z$,
and the generated distribution of subgroup $z$, paralleling standard GAN results~\cite{gan}.

\begin{restatable}{lmm}{lmmdiscriminator}
\label{lmm:discriminator}
The optimal  %
discriminator between original subgroup distribution $P_z$ and augmented subgroup
$P_{\tilde{z}}$
has loss
$\mathcal{L}_{CG}^* = \mathbb{E}_{[x] \sim [X]} JS([x]_z, [\xt]_z) - \log 2$.
\end{restatable}
\begin{proof}[Proof of \cref{lmm:discriminator}]
  By Proposition~\ref{prop:js-disc},
  \begin{align*}
    \E_{[x] \sim [X]} JS([x]_z, [x]_{\tilde{z}})
    = \log 2
    + \frac{1}{2} \E_{[x] \sim [X]} \log D_{[x]}^z([x]_z)
    + \frac{1}{2} \E_{[x] \sim [X]} \log (1 - D_{[x]}^z([\xt]_{z}))
  \end{align*}
  where $D_{[x]}^z$ is a discriminator for this coupled set (within subgroup $z$).
  Instead of training a separate discriminator per example or coupled set, it is enough to train a single discriminator $D$ conditioned on this specific coupled set $([x]_z, [x]_{\tilde{z}})$.
  In other words this is a discriminator whose input is both the original example $[x]_z$ and a generated version $[x]_{\tilde{z}}$, and 
  for each input guesses its chance of being a real example.
  This is exactly the pair-conditioned discriminator described in Section~\ref{sec:cg_background}.
\end{proof}

\paragraph{\bf Proof of \cref{thm:main}.} We finally put the pieces together to prove the main result, restated here for convenience.
\thmmain*

First, the equivalence of the quantity we care about $I(Z; \hat{Y}; [X])$ and the consistency loss on true coupled sets is given by \cref{lmm:conditional-mi}.
It remains to bound $\E JS(f_\theta([x]_1), f_\theta([x]_2))$, which can be bounded by the consistency loss on augmented examples
$\E JS(f_\theta([\tilde{x}]_1), f_\theta([\tilde{x}]_2))$
and the optimal CycleGAN losses
$\E JS(f_\theta([x]_i), f_\theta([\tilde{x}]_i))$
by metric properties of the JSD.

\begin{proof}[Proof of \cref{thm:main}]
  Consider any fixed subgroup $z$ and let $\bar{X}_z$ denote the R.V. from the mixture distribution of $P_z$ and $P_{\tilde{z}}$, i.e.\ either a true example or an augmented example from subgroup $z$.
  Let $W$ denote the (binary) indicator of this mixture.
  Then
  \begin{equation}
    \label{eq:data-processing}
    JS(f_\theta([x]_z), f_\theta([\tilde{x}]_z)) = I(W; f_\theta(\bar{X}_z)) \le I(W; \bar{X}_z) = JS([x]_z, [\tilde{x}]_z),
  \end{equation}
  where the equalities are Proposition~\ref{prop:js-mi}
  and the inequality is an application of the \emph{data processing inequality} on the Markov chain $W \to \bar{X}_z \to f_\theta(\bar{X}_z)$.

  Combining equation~\eqref{eq:data-processing} with \cref{lmm:discriminator},
  applying the definition of $\mathcal{L}_{CG}^{z}$,
  and summing over two groups $z=1, z=2$ yields
  \begin{equation}
    \label{eq:thm-1}
    \begin{aligned}
      & JS(f_\theta([x]_1), f_\theta([\tilde{x}]_1))^{\frac{1}{2}} + JS(f_\theta([x]_2), f_\theta([\tilde{x}]_2))^{\frac{1}{2}}
      \\
      & \le \mathcal{L}_{CG}^{z_1}(x; \theta)^{\frac{1}{2}}
      + \mathcal{L}_{CG}^{z_2}(x; \theta)^{\frac{1}{2}}
    \end{aligned}
  \end{equation}

  By definition of the self-consistency loss \eqref{eq:self-consistency}
  and \cref{def:js},
  \begin{equation}
    \label{eq:thm-2}
    JS(f_\theta([\tilde{x}]_1), f_\theta([\tilde{x}]_2)) = \mathcal{L}_s(x, [\tilde{x}]; \theta),
  \end{equation}
  for any sample $x$ and where $[\tilde{x}]$ denotes the generated coupled set $\{F_1(x), F_2(x)\}$ as usual.
  Denoting the right hand side $\mathcal{L}_s(x; \theta)$ for shorthand,
  summing equations~\eqref{eq:thm-1} and~\eqref{eq:thm-2},
  and using the metric property of the JSD (\cref{prop:js-metric}) gives
  \begin{align*}
    JS(f_\theta([x]_1), f_\theta([x]_2))^{\frac{1}{2}}
    \le \mathcal{L}_s(x; \theta)^{\frac{1}{2}}
    +\mathcal{L}_{CG}^{z_1}(x; \theta)^{\frac{1}{2}}
    + \mathcal{L}_{CG}^{z_2}(x; \theta)^{\frac{1}{2}}.
  \end{align*}
  Finally, squaring and averaging over the dataset and applying \cref{lmm:conditional-mi} gives the result of \cref{thm:main}:
  \begin{align*}
    I(\Yh; Z \mid [X])
    \le
    \E_{x \sim X} \left( \mathcal{L}_{s}(x; \theta)^{\frac{1}{2}} + \mathcal{L}_{CG}^{z_1}(x; \theta)^{\frac{1}{2}} + \mathcal{L}_{CG}^{z_2}(x; \theta)^{\frac{1}{2}}\right)^2.
  \end{align*}
\end{proof}

These pieces can be combined to show that the GAN-based modeling of subgroups (Stage 1) and the consistency regularizer (Stage 2) together minimize the desired identity-conditioned mutual information,
which completes the proof of Theorem~\ref{thm:main}.

%% file: src/analysis_background.tex
\subsection{Background: CycleGAN}
\label{sec:cg_background}

Given two groups $A$ and $B$, CycleGAN learns mappings $F : B \to A$ and $G : A \to B$ given unpaired samples $a \sim P_{A}, b \sim P_{B}$.
Along with these generators, it has adversarial discriminators $D_A, D_B$ trained with the standard GAN objective,
i.e.\ $D_A$ distinguishes samples $a \sim P_{A}$ from generated samples $F(b)$, where $b \sim P_{B}$. In \methodabbrv, $A$ and $B$ correspond to data from a pair of subgroups $z, z'$ of a class.

CycleGAN uses a \emph{cycle consistency loss} to ensure that the mappings $F$ and $G$ are nearly inverses of each other,
which biases the model toward learning meaningful cross-domain mappings.
An additional \emph{identity loss} is sometimes used which also encourages the maps $F, G$ to preserve their original domains \emph{i.e.} $F(a) \approx a$ for $a \sim P_A$.
These cycle consistency and identity losses can be modeled by respectively minimizing $\mathcal{L}_{CG}(a, F(G(a)))$ and $\mathcal{L}_{CG}(a, F(a))$
for some function $\mathcal{L}_{CG}$ which measures some notion of distance on $A$ (with analogous losses for $B$). Figure~\ref{fig:cyclegan} visualizes the CycleGAN model.

\begin{wrapfigure}{R}{0.45\textwidth}%
    \centering
    \includegraphics[width=0.8\linewidth]{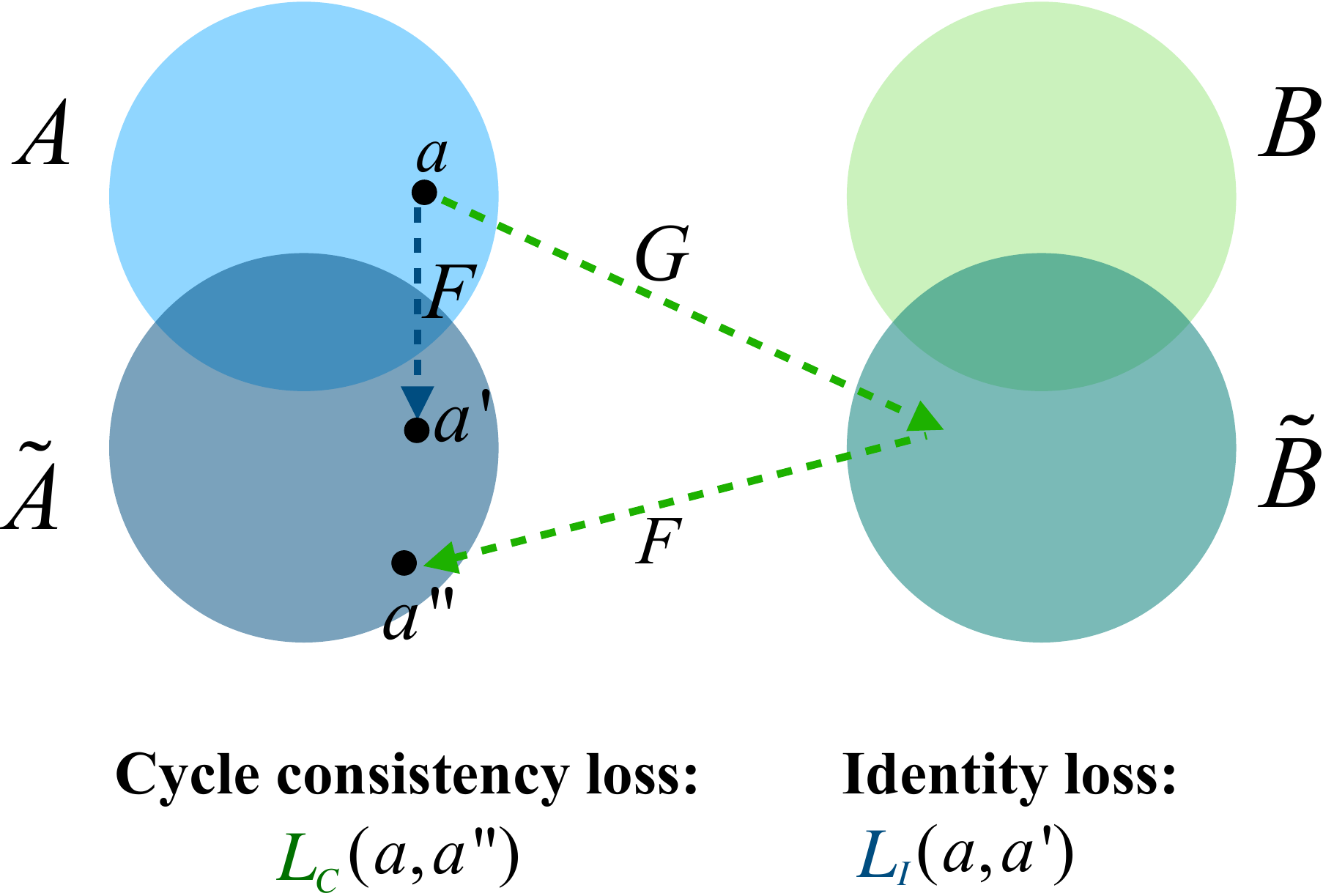}
    \caption{\small CycleGAN learns mappings on domains $A \cup B$, where $F$ maps examples to $A$ and $G$ maps to $B$. To model possible distribution shift introduced by the generative model, we denote their images as ${\rm Im}(F) = \tA, {\rm Im}(G) = \tB$ respectively. Semantically consistent mappings are encouraged with the cycle consistency and identity losses, e.g. to ensure that $F(a) = a$ for all $a \in A$.}
    \label{fig:cyclegan}
\end{wrapfigure}

\begin{defn}
\label{def:cyclegan-loss}
The sum of the CycleGAN cycle consistency $\mathcal{L}_{CG}(a, F(G(a))$ and identity $\mathcal{L}_{CG}(a, F(a))$ losses on domain $A$ is denoted
$\mathcal{L}_{CG}^A(a; \theta)$ for overall CycleGAN parameters $\theta$, and similarly for domain $B$.
In the context of Stage 1 of model patching, let
$\mathcal{L}_{CG}^z(x; \theta)$ denote the loss when the domain is one of the subgroups $z$.
\end{defn}

The original CycleGAN uses the $\ell_1$ distance $L(a, \at) = \|a - \at\|_1$.
However, we note that many other functions can be used to enforce similarity.
In particular, we point out that a \emph{pair-conditioned discriminator} $\mathcal{D}\{a, \at\} \mapsto [0,1]^2$ can also be used,
which accepts a coupled pair of original and translated examples
and assigns a probability to each of being the original example.
If the guesses for the true and translated examples are $\mathcal{D}_a$ and $\mathcal{D}_{\tilde{a}}$ respectively,
then the distance is
$
L(a, \tilde{a}) = \max_{\mathcal{D}} \log \mathcal{D}_a + \log(1 - \mathcal{D}_{\tilde{a}}) + \log 2.
$
To sanity check that this has properties of a distance,
note that $L$ decreases as $a, \tilde{a}$ are more similar, as the discriminator has trouble telling them apart.

Intuitively, the discriminator loss is a measure of how similar the original and generated distributions are,
which will be used in Section~\ref{sec:mi-consistency} to prove our main result.

\subsection{Background: Properties of the Jensen-Shannon Divergence}
\label{sec:js}

We define the Jensen-Shannon divergence (JSD) and its properties that will be used in our method and analysis.

\begin{defn}
  \label{def:js}
  The Jensen-Shannon Divergence (JSD) of distributions $P_1, \dots, P_k$ is $JS(P_1, \dots, P_k) = \frac{1}{k} \sum_{i=1}^k {\rm KL}(P_i \| M)$ where $M = \frac{1}{k} \sum_{i=1}^k P_i$.

  We overload the $JS(\cdot)$ function in the following ways.
  The JSD of random variables $X_1, \dots, X_k$ is the JSD of their laws (distributions).

  Additionally, we define the JSD of vector-valued inputs if they represent distributions from context.
  For example, for a model $f$ that outputs a vector representing a categorical distribution,
  $JS(f_\theta(x_1), \dots, f_\theta(x_k))$ is the JSD of those distributions.
\end{defn}

We briefly review important properties of the JSD. 
Unlike the KL divergence and other notions of distributional distance, the JSD can be related to a metric.
\begin{proposition}
  \label{prop:js-metric}
  The JSD is the square of a metric. In particular, any three distributions $p, q, r$ satisfy
  $JS(p, q)^{1/2} + JS(q, r)^{1/2} \ge JS(p, r)^{1/2}$.
\end{proposition}

Finally, the following fact about the JSD relating it to the mutual information of a mixture distribution and its indicator variable will be useful in our analysis.
\begin{proposition}
  \label{prop:js-mi}
  Let $Z$ be a uniform categorical indicator variable with support $[k]$ and $P_i, i\in [k]$ be distributions.
  Let $X \sim P_z, z \sim Z$ be the random variable associated with the mixture distribution of the $P_i$ controlled by the indicator $Z$.
  Then $I(X; Z) = JS(P_1, \dots, P_k)$.
\end{proposition}

Finally, we review
standard results (e.g., from the GAN literature) on the relationship between discriminators and the JS divergence,
which relates the loss of an optimal discriminator to the JSD of the two distributions.
We include a proof for completeness.

\begin{proposition}
  \label{prop:js-disc}
  Consider two domains $A$ and $\tilde{A}$ (i.e., distributions on a common support $\mathcal{A}$),
  with densities $p(a), \tilde{p}(a)$ respectively.
  Consider a discriminator $D : \mathcal{A} \to \mathbb{R}$ optimized to maximize the loss
  \[
    \mathcal{L}(D)
    = \frac{1}{2} \mathbb{E}_{a \sim p(a)} \log D(a) + \frac{1}{2} \mathbb{E}_{a \sim \tilde{p}(a)} \log (1-D(a)).
  \]
  Then the value of this loss for the optimal discriminator $D^*$ is $JS(A, \tilde{A}) - \log 2$.
\end{proposition}
\begin{proof}
  Differentiate the loss with respect to the discriminator's output $D(a)$ for any example $a \in \mathcal{A}$,
  which yields
  \begin{align*}
    \frac{1}{2} p(a) \frac{1}{D(a)} - \frac{1}{2} \tilde{p}(a) \frac{1}{1-D(a)}.
  \end{align*}
  The loss is maximized at $D^*(a) = \frac{p(a)}{p(a) + \tilde{p}(a)}$.
  The result follows from plugging this discriminator into the loss and using \cref{def:js}:
  \begin{align*}
    \mathcal{L}(D^*)
    &= \frac{1}{2} \mathbb{E}_{a \sim p(a)} \log \frac{p(a)}{p(a) + \tilde{p}(a)} + \frac{1}{2} \mathbb{E}_{a \sim \tilde{p}(a)} \frac{\tilde{p}(a)}{p(a) + \tilde{p}(a)}
    \\
    &= \frac{1}{2} KL\left(A \| \frac{A + \tilde{A}}{2}\right)
    + \frac{1}{2}  KL\left(\tilde{A} \| \frac{A + \tilde{A}}{2}\right)
    \\&\qquad - \log(2)
    \\
    &= JS(A, \tilde{A}) - \log 2.
  \end{align*}
\end{proof}

%% file: src/experiment_details.tex
\section{Experimental Details}
\label{sec:experimental_details}
We provide detailed information about our experimental protocol and setup for reproducibility, including dataset information in \ref{sec:full_dataset_info}, 

\subsection{Dataset Information}
\label{sec:full_dataset_info}

\begin{table*}
\caption{Number of training, validation and test examples in each dataset.}
\label{tab:dataset-size}
\centering
\resizebox{\textwidth}{!}{%
\begin{tabular}{@{}llllll@{}}
\toprule
\multicolumn{1}{c}{\multirow{1}{*}{\textbf{Dataset}}} & \multirow{1}{*}{\bf Split} & \multicolumn{4}{c}{\textbf{Subgroup Size} $(Y, Z)$} \\ %
\midrule

\multirow{4}{*}{\bf MNIST-Correlation} & & even, clean &  even, zigzag &  odd, clean &  odd, zigzag\\\cmidrule{3-6}
         & train                   & 9900        & 100         & 100        & 9900       \\
  & validation                   & 9900        & 100         & 100        & 9900       \\
 & test                   & 4926        & 4926         & 5074        & 5074       \\
\midrule
\multirow{4}{*}{\bf Waterbirds} & & landbird, land &  landbird, water &  waterbird, land &  waterbird, water\\\cmidrule{3-6}
  &            train                    & 3498        & 184         & 56         & 1057       \\
   &            validation                    & 467        & 466         & 133         & 133       \\
   &            test                    & 2255        & 2255         & 642         & 642       \\
\midrule
\multirow{4}{*}{\bf CelebA-Undersampled} & & non-blonde, female &  non-blonde, male &  blonde, female &  blonde, male\\\cmidrule{3-6}
&     train                & 4054        & 66874       & 22880      & 1387       \\
&         validation            & 8535        & 8276       & 2874      & 182       \\
&        test             & 9767        & 7535       & 2480      & 180       \\
\midrule
\multirow{4}{*}{\bf ISIC} & & benign, no bandage &  benign, bandage &  malignant, no bandage &  malignant, bandage\\\cmidrule{3-6}
&      train           & 8062        & 7420        & 1843       & 0          \\
&      validation      & 1034        & 936        & 204       & 0          \\
& test                 & 1026        & 895        & 239       & 0          \\ \bottomrule
\end{tabular}%
}
\end{table*}
We provide details for preprocessing and preparing all datasets in the paper. Table~\ref{tab:dataset-size} summarizes the sizes of the subgroups present in each dataset.
All datasets will be made available for download.

\paragraph{\bf MNIST-Correlation.} We mix data from MNIST~\cite{lecun1998gradient} and MNIST-Corrupted~\cite{mu2019mnist} to create a controlled setup. We classify digit parity $Y \in \{ {\rm even}, {\rm odd}\}$, where each class is divided into subgroups $Z \in \{{\rm clean}, {\rm zigzag}\}$, drawing digits from MNIST and MNIST-Corrupted (with the zigzag corruption) respectively. 

To generate the dataset, we use the following procedure: 
\begin{itemize}
    \item Fix a total dataset size $N$, and a desired correlation $\rho$.
    \item Sample 
    \begin{itemize}
        \item $\left\lfloor{\frac{(\rho + 1) N}{4}}\right\rfloor$ ${\rm even}$ digits from MNIST
        \item $\frac{N}{2} - \left\lfloor{\frac{(\rho + 1) N}{4}}\right\rfloor$ ${\rm even}$ digits from MNIST-Corrupted
        \item $\frac{N}{2} - \left\lfloor{\frac{(\rho + 1) N}{4}}\right\rfloor$ ${\rm odd}$ digits from MNIST
        \item $\left\lfloor{\frac{(\rho + 1) N}{4}}\right\rfloor$ ${\rm odd}$ digits from MNIST-Corrupted
    \end{itemize}
    
\end{itemize}
This generates a dataset with balanced $Y$ and $Z$ with size $\frac{N}{2}$ each. For our experiments, we use $N=40000, \rho=0.98$. This makes $Y$ and $Z$ highly correlated, so that most ${\rm even}\,\, ({\rm odd})$ digits are ${\rm clean}\,\, ({\rm zigzag})$. For validation, we use $50\%$ of the training data.

\begin{figure}[b!!]%
    \centering
    \includegraphics[width=0.4\linewidth]{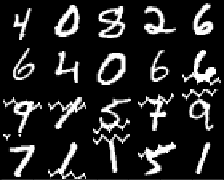}
    \caption{An example of data in MNIST-Correlation. Most ${\rm even}$ digits are ${\rm clean}$ while most ${\rm odd}$ digits contain a ${\rm zigzag}$ corruption.}
    \label{fig:mnist-correlation-example}
\end{figure}

\paragraph{\bf CelebA-Undersampled.} We modify the CelebA dataset \cite{liu2015deep} by undersampling the $({\rm Y=non\text{-}blonde,\, Z=female})$ subgroup in the training set. The original dataset contains $71629$ examples in this training subgroup, and we keep a random subset of $4054$ examples. This number is chosen to make the ratio of subgroup sizes equal in both classes $\left(\frac{4054}{66874} \approx \frac{1387}{22880}\right)$. We do not modify the validation or test datasets. 

This modification introduces a spurious correlation between hair-color and gender, which makes the dataset more appropriate for our setting. We preprocess images by resizing to $128 \times 128 \times 3$ before use.

\paragraph{\bf Waterbirds.} We use the Waterbirds dataset \cite{sagawa2019distributionally} and resize images to $224\times224\times3$ before use. Note that this differs from  the preprocessing used by \cite{sagawa2019distributionally}, who first resize to $256\times256\times3$ and then center-crop the image to $224\times224\times3$. The preprocessing they use makes the task easier, since some part of the (spurious) background is cropped out, while we retain the full image.

\paragraph{\bf ISIC.} We use the ISIC dataset \cite{codella2018skin} and resize images to $224 \times 224 \times 3$ before use.

\subsection{CycleGAN Training Details}
\label{sec:cyclegan_training_details}
We use the default hyperparameters suggested by \cite{zhu2017unpaired} for CycleGAN training, with batchnorm for layer normalization. We use Adam for optimization $(\beta_1 = 0.5)$ with a constant learning rate of $0.0002$ for both generators and both discriminators. 

\paragraph{\bf MNIST-Correlation.} Train on $200$ images each from both MNIST and MNIST-Corrupted ($100$ images per class) for $2500$ epochs with a batch size of $25$, cycle loss coefficient of $10.0$ and identity loss coefficient of $1.0$. We randomly rotate, pad and crop every image for training.

\paragraph{\bf CelebA-Undersampled.} Train separate CycleGANs for both classes. Train on $1000$ images each from both subgroups within the class for $4000$ epochs with a batch size of $16$, cycle loss coefficient of $10.0$ and identity loss coefficient of $1.0$. We flip inputs randomly (with probability $0.5$) and randomly crop up to $10\%$ of every image. Due to instability during training, we visually inspected samples generated on the training set at several checkpoints to pick the best model.

\paragraph{\bf Waterbirds.} Train separate CycleGANs for both classes. Train on $56$ and $184$ images each from both subgroups for the ${\rm landbird}$ and ${\rm waterbird}$ classes respectively. Train for $4000$ epochs with a batch size of $4$, cycle loss coefficient of $10.0$ and identity loss coefficient of $1.0$. We flip inputs randomly (with probability $0.5$) and randomly crop upto $10\%$ of every image.

\paragraph{\bf ISIC.} Train on $100$ images each from both benign subgroups (with and without bandaids) for $4000$ epochs with a batch size of $4$, cycle loss coefficient of $10.0$ and identity loss coefficient of $10.0$. We flip inputs randomly (with probability $0.5$) and randomly crop upto $10\%$ of every image. %

\subsection{Architectures and Training Information}
All training code is written in Python with tensorflow-2.0. All models are trained with Stochastic Gradient Descent (SGD), with a momentum of 0.9. In order to isolate the effect of our method, we do not use any data augmentation (such as pad and crop operations or random flips) when training the classifier.

\paragraph{\bf MNIST-Correlation.} We train a convolutional neural network from scratch, initialized with random weights. The architecture is provided below,

${\rm Conv2D(filters\!=\!32, kernel\!=\!3)} \rightarrow {\rm ReLU} \rightarrow {\rm Conv2D}(32, 3) \rightarrow {\rm ReLU} \rightarrow {\rm MaxPooling2D}({\rm pooling\!=\!2})
\\
\rightarrow {\rm Dropout}({\rm p\!=\!0.25}) \rightarrow {\rm Conv2D}(64, 3)  \rightarrow {\rm ReLU} \rightarrow {\rm Conv2D}(64, 3) \rightarrow {\rm ReLU} \rightarrow {\rm MaxPooling2D}(2) \rightarrow {\rm Dropout}(0.25) \\
\rightarrow {\rm Flatten}
\rightarrow {\rm Dense}({\rm units\!=\!64}) \rightarrow {\rm ReLU} \rightarrow {\rm Dropout}(0.5) \rightarrow {\rm Dense}(10) \rightarrow {\rm Softmax}$.

\paragraph{\bf Other datasets.} All models are fine-tuned using a ResNet-50 architecture, with pretrained ImageNet weights\footnote{The particular model used was taken from \url{https://github.com/qubvel/classification_models}.}. The only preprocessing common to all methods is standard ImageNet normalization using $\mu = [0.485, 0.456, 0.406], \sigma = [0.229, 0.224, 0.225]$.

\subsection{Hyperparameters}

For model selection, we use robust accuracy on the validation set\footnote{For the ISIC dataset, we additionally performed model selection using AUROC, as illustrated in Table~\ref{tab:isic}.}. The selected model's hyperparameters are then run 3 times, and the results averaged over these trials are reported in Table~\ref{tab::results-comparison-with-other-methods}. Below, we provide details of all hyperparameter sweeps, and in Table~\ref{tab::best-hyperparameters}, we include the best hyperparameters found for each method and dataset.

\subsubsection{CelebA-Undersampled}

We run sweeps for all methods over $50$ epochs.

\paragraph{\bf ERM.} Sweep over learning rates $\{0.0001, 0.00005, 0.00002, 0.00001 \}$ with weight decay fixed to $0.05$.

\paragraph{\bf GDRO.} Sweep over adjustment coefficients in $\{1.0, 3.0 \}$ and learning rates $\{0.0001, 0.00005 \}$ with weight decay fixed to $0.05$. 

\paragraph{\bf \methodabbrv.} Sweep over consistency penalties in $\{5.0, 10.0, 20.0, 50.0 \}$. Learning rate is fixed to $0.00005$, weight decay fixed to $0.05$ and the adjustment coefficient is fixed to $3.0$.

\subsubsection{Waterbirds}

We run sweeps for all methods over $500$ epochs.

\paragraph{\bf ERM.} Sweep over learning rates $\{0.001, 0.0001, 0.00001\}$ and weight decays $\{0.5, 0.001\}$.

\paragraph{\bf GDRO.} Sweep over learning rates $\{0.00001, 0.00005 \}$ and weighte decays $\{0.5, 0.05 \}$ with adjustment coefficient fixed to $1.0$ and batch size $24$. We also separately swept weight decays $\{1.0, 0.001 \}$ and adjustment coefficients over $\{ 1.0, 2.0\}$.

\paragraph{\bf \methodabbrv.} Sweep over consistency penalties in $\{100.0, 200.0 \}$ and learning rates $\{0.00005, 0.0001\}$. Weight decay fixed to $0.001$ and adjustment coefficient is fixed to $2.0$. Separately, we sweep over learning rates $\{0.00001, 0.00002, 0.00005, 0.0001\}$, fixing the consistency penalty to $200.0$, weight decay to $0.05$ and adjustment coefficient to $1.0$.

\subsubsection{MNIST-Correlation}

We run sweeps for all methods over 100 epochs. 

\paragraph{\bf ERM.} Sweep over learning rates $\{0.0001, 0.0002, 0.0005, 0.001\}$ and weight decays $\{0.0005, 0.05\}$.

\paragraph{\bf GDRO.} Sweep over learning rates $\{0.0001, 0.0002, 0.0005, 0.001\}$ and weight decays $\{0.0005, 0.05\}$. Adjustment coefficient is fixed to $1.0$.

\paragraph{\bf CDAT.} Sweep over domain loss coefficients $\{-0.1, -0.01, 0.1, 1.0 \}$. We fix learning rate to $0.001$ and weight decay to $0.0005$. We run CDAT for $400$ epochs, since it takes much longer to converge.

\paragraph{\bf IRM.} Sweep over IRM penalty $\{0.01, 0.1, 1.0, 10, 100, 1000, 10000\}$ and learning rates $\{0.0005, 0.001\}$. Weight decay is fixed to $0.0005$.

\paragraph{\bf \methodabbrv.} Sweep over consistency penalty weights $\{0.0, 2.0, 5.0, 10.0, 50.0\}$. Learning rate is fixed to $0.001$ and weight decay is fixed to $0.0005$.

\subsubsection{ISIC}
We run sweeps for all methods over 75 epochs.

\paragraph{\bf ERM.} Sweep over weight decays $\{0.5, 0.05, 0.00005\}$. Learning rate is fixed to $0.0001$.

\paragraph{\bf GDRO.} Sweep over learning rates $\{0.0001, 0.00001\}$ and weight decays $\{0.5, 0.05, 0.00005\}$. Adjustment coefficient is fixed to $0$.

\paragraph{\bf \methodabbrv.} Sweep over learning rates $\{0.0001, 0.00005\}$, weight decays $\{0.01, 0.05\}$, consistency penalties $\{10.0, 50.0\}$ and annealing rates $\{0.005, 0.002\}$.

\subsection{Mutual Information Measurement}
\label{sec:mi_measurement}
For the mutual information measurement experiment on MNIST-Correlation in Section~\ref{sec:benchmarks}, we additionally attach a domain prediction head to the final feature layer. This domain prediction head is then used to predict the subgroup $z$ of any example $x$. Note that this domain prediction head does not pass back gradients to the main model, it merely observes the learned representation and attempts to improve prediction accuracy of the subgroups using this. Intuitively, this captures how much information about the subgroups is available to be ``squeezed-out'' by the domain prediction head. This constitutes a use of Lemma~\ref{lmm:mi-adversarial} to estimate the mutual information, and we report the average cross-entropy loss (added to $\log 2$).

\subsection{Baseline Comparisons}
\label{sec:baselines}
We describe the baselines that we compare to, with implementations for each of these available in our code release.

\subsubsection{Methods}

\paragraph{\bf ERM.} We use standard training with a cross-entropy loss. ERM cannot take advantage of knowledge of the subgroups, so this constitutes a standard baseline that a practitioner might use to solve a task.

\paragraph{\bf GDRO.} This is our main baseline as described in Section~\ref{sec:method}, and uses a stochastic optimization method~\cite{sagawa2019distributionally}. GDRO uses subgroup information to optimize the worst-case loss over all subgroups. We note that GDRO requires the specification of an adjustment coefficient, and we describe the best found coefficients in Table~\ref{tab::best-hyperparameters}.

\paragraph{\bf CDAT.} We use a generic domain adversarial training approach using a domain prediction head attached to the last feature layer of the model $\phi(X)$. The domain head predicts the subgroup identity of the given example, and we use gradient reversal in order to erase domain information from the representation $\phi(X)$. We vary the magnitude of the gradient reversal on the domain loss (which we call the domain loss coefficient in Table~\ref{tab::best-hyperparameters}) in order to find the best-performing model.

\paragraph{\bf IRM.} We implement the IRM penalty~\cite{arjovsky2019invariant}, and treat the subgroups as separate environments across which the model should perform well.

\subsubsection{Ablations}

\paragraph{\bf Subgroup Pairing.} We simply take pairs of examples that lie in different subgroups and enforce consistency on them.

\paragraph{\bf Heuristic Augmentations.} We build a pipeline inspired by AugMix~\cite{hendrycks2019augmix} using the following operations: shearing, translation, rotation, flipping, contrast normalization, pixel inversion, histogram equalization, solarization, posterization, contrast adjustment, color enhancement, brightness adjustment, sharpness adjustment, cutout and mixup. We sample between 1 and 3 of these augmentations in a random order and apply them to the image.

\subsection{ISIC Spurious Correlations}
\label{sec:isic_spurious}
For completeness, we include a detailed evaluation for the ISIC dataset in Table~\ref{tab:isic_full}. Here, we highlight that regardless of what criterion is used for model selection between robust accuracy and AUROC, \methodabbrv\ exceeds the performance of the other methods.

For ISIC, we also create an alternate evaluation dataset with artificial images in order to test whether a model spuriously correlates the presence of a bandage with the benign cancer class. To construct this dataset, we use image segmentation to automatically extract images of the bandages from the benign cancer class, and superimpose them on images with malignant cancers. This allows us to generate the artificial subgroup of the malignant cancer class that would contain images with bandages. We use this dataset to highlight how \methodabbrv\ improves the model's dependence on this spurious feature in Figure~\ref{fig:isic}.

\begin{table}
\centering
\caption{\small Performance on the ISIC validation set.}
\resizebox{0.5\linewidth}{!}{%
\begin{tabular}{llll}
\toprule
{\bf Evaluation} & \multirow{2}{*}{\bf Method} &\multicolumn{2}{c}{ Model Selection Criterion}   \\
{\bf Metric} & & \textbf{Robust Acc.} & \textbf{AUROC}\\
 \midrule
 \textbf{Robust} & ERM & 65.59 (1.17) & 52.93 (10.27) \\   
  \textbf{Acc.} & GDRO & 64.97 (3.15) & 51.23 (1.93) \\   
     & CAMEL & {\bf 77.45} (0.35)  & {\bf 66.67} (3.03) \\ 
     \midrule 
 \textbf{AUROC} & ERM & {\bf 92.48} (0.80)  & {\bf 93.38} (0.14) \\
 & GDRO &   89.50 (2.50) & 91.83 (0.11)\\
 & CAMEL & {\bf 92.47} (0.38)  & {\bf 93.41} (0.52) \\
 \bottomrule
\end{tabular}%
}
\label{tab:isic_full}

\end{table}

\subsection{Alternative GAN Augmentation Baselines}
\label{sec:gan_baselines}
As noted in Section~\ref{sec:cyclegan}, Stage 1 of the model patching pipeline can be integrated with alternative domain translation models.
As an additional baseline, we compare to alternative GAN augmentation methods.
Typically, these methods are used as a data augmentation method, but not evaluated on robustness.

We consider the Augmented CycleGAN~\cite{almahairi2018augmented}, Data Augmentation GAN (DAGAN)~\cite{antoniou2017data} and StarGAN-v2~\cite{choi2020stargan} models, either when used in combination with ERM, or when as a part of the model patching baseline.
When used as a part of model patching, we replace the CycleGAN in Stage 1 with the alternative GAN model.

We used released code for Augmented CycleGAN and DAGAN to generate data for the Waterbirds dataset.
For StarGANv2, we used pre-trained models for Celeb-A.
We note that DAGAN is meant to be a self-contained data augmentation pipeline, so we did not consider it in conjunction with Model Patching. 

The results of this comparison is are shown in~\ref{tab:gan_baselines}.
In particular, these alternate models have poor robust performance when used purely for data augmentation.
Their performance improves when integrated in the model patching pipeline.

\begin{table}%
\centering
\caption{\small Comparisons to GAN Baselines on Waterbirds and CelebA-Undersampled.}
\resizebox{0.6\linewidth}{!}{%
\begin{tabular}{@{}llll@{}}
  \toprule
  \multicolumn{1}{c}{\multirow{2}{*}{{\bf Dataset}}} & \multicolumn{1}{c}{\multirow{2}{*}{{\bf GAN Model}}} & \multicolumn{2}{c}{{\bf Robust/Aggregate Acc.}} \\ \cmidrule(l){3-4}
  \multicolumn{1}{c}{} & \multicolumn{1}{c}{} & GAN + ERM   & GAN + Model Patching \\ \midrule
  Waterbirds           & CycleGAN            & 76.88/91.75 & \textbf{89.12}/90.89      \\
                       & Augmented\ CycleGAN   & 63.12/91.08 & \textbf{84.87}/86.44          \\
                       & DAGAN                & 73.12/90.28 & {---}                  \\ \midrule
  CelebA-Undersampled               & StarGAN v2           & 65.91/90.58 & \textbf{80.68}/89.33          \\ \bottomrule
\end{tabular}
}
\label{tab:gan_baselines}
\end{table}

\begin{table*}[t!]
\caption{The values of the best hyperparameters found for each dataset and method. }
\label{tab::best-hyperparameters}
\centering
\resizebox{\textwidth}{!}{%
\begin{tabular}{@{}lllllll@{}}
\toprule
\multirow{1}{*}{\bf Method} & \multirow{1}{*}{\bf Dataset} & \multicolumn{5}{c}{\bf Hyperparameters}                                             \\ \midrule
                        &                          & {\bf Learning Rate} & {\bf Weight Decay} & {\bf Batch Size} &                         &           \\ \cmidrule(l){3-7}
\multirow{4}{*}{\bf ERM}                     & {\bf MNIST-Correlation}        & 0.0001        & 0.05         &    100        &                         &           \\
                        & {\bf CelebA-Undersampled}      & 0.00005       & 0.05         &       16     &                         &           \\
                        & {\bf Waterbirds}               & 0.001         & 0.001        & 16         &                         &           \\
                        & {\bf ISIC}                     & 0.0001          & 0.005              & 24           &                         &           \\ 
                        &                   & 0.0001          & 0.00005              & 24           &                         &           \\ \midrule
                        &                          & {\bf Learning Rate} & {\bf Weight Decay} & {\bf Batch Size} & {\bf GDRO Adjustment}         &           \\ \cmidrule(l){3-7} 
\multirow{4}{*}{\bf GDRO}                    & {\bf MNIST-Correlation}        & 0.0005        & 0.0005       &      100     & 1.0                     &           \\
                        & {\bf CelebA-Undersampled}      & 0.0001        & 0.05         &    16       & 3.0                     &           \\
                        & {\bf Waterbirds}               & 0.00001       & 0.05         & 24         & 1.0                     &           \\
                        & {\bf ISIC}                     &   0.0001            & 0.05             &   24         &          0.0               &           \\ 
                        &                  &   0.0001            & 0.00005             &   24         &          0.0               &           \\ \midrule
                        &                          & {\bf Learning Rate} & {\bf Weight Decay} & {\bf Batch Size} & {\bf GDRO Adjustment}         & $\mathbb{\lambda}$ \\ \cmidrule(l){3-7} 
\multirow{4}{*}{\bf \methodabbrv}                   & {\bf MNIST-Correlation}        & 0.001         & 0.0005       &    100        & 1.0                     & 5.0       \\
                        & {\bf CelebA-Undersampled}      & 0.00005       & 0.05         &       16     & 3.0                     & 5.0       \\
                        & {\bf Waterbirds}               & 0.0001        & 0.001        & 16         & 2.0                     & 100.0     \\
                        & {\bf ISIC}                     &     0.0001          &   0.01           &     24       &          3.0               &  50.0\footnotemark         \\
                        &                 &     0.0001          &       0.01       &     24       &          3.0               &  10.0\footnotemark[2]         \\\midrule
                        &                          & {\bf Learning Rate} & {\bf Weight Decay} & {\bf Batch Size} & {\bf Domain Loss Coefficient} &           \\ \cmidrule(l){3-7} 
{\bf CDAT}                    & {\bf MNIST-Correlation}        & 0.001         & 0.0005       &     100       & -0.10                   &           \\ \midrule
                        &                          & {\bf Learning Rate} & {\bf Weight Decay} & {\bf Batch Size} & {\bf IRM Anneal Steps}             & {\bf IRM Penalty}          \\ \cmidrule(l){3-7} 
{\bf IRM}                     & {\bf MNIST-Correlation}        &     0.0005          &   0.0005           &     100       &     2000                    &  0.1         \\ \bottomrule
\end{tabular}%
}
\end{table*}
\footnotetext[2]{The consistency penalty is increased linearly on every step, from 0 to $\lambda$ with rates $0.002$ and $0.005$ for $\lambda=50.0$ and $\lambda=10.0$ respectively.}

%% file: main.bbl
\begin{thebibliography}{10}

\bibitem{10.1001/jamadermatol.2018.2348}
A.~S. Adamson and A.~Smith.
\newblock {Machine Learning and Health Care Disparities in Dermatology}.
\newblock {\em JAMA Dermatology}, 154(11):1247--1248, 11 2018.

\bibitem{almahairi2018augmented}
A.~Almahairi, S.~Rajeswar, A.~Sordoni, P.~Bachman, and A.~Courville.
\newblock Augmented cyclegan: Learning many-to-many mappings from unpaired
  data.
\newblock {\em arXiv preprint arXiv:1802.10151}, 2018.

\bibitem{antoniou2017data}
A.~Antoniou, A.~Storkey, and H.~Edwards.
\newblock Data augmentation generative adversarial networks.
\newblock {\em arXiv preprint arXiv:1711.04340}, 2017.

\bibitem{arjovsky2019invariant}
M.~Arjovsky, L.~Bottou, I.~Gulrajani, and D.~Lopez-Paz.
\newblock Invariant risk minimization.
\newblock {\em arXiv preprint arXiv:1907.02893}, 2019.

\bibitem{baird1992document}
H.~S. Baird.
\newblock Document image defect models.
\newblock In {\em Structured Document Image Analysis}, pages 546--556.
  Springer, 1992.

\bibitem{Baluja2017AdversarialTN}
S.~Baluja and I.~C. Fischer.
\newblock Adversarial transformation networks: Learning to generate adversarial
  examples.
\newblock {\em ArXiv}, abs/1703.09387, 2017.

\bibitem{Beery2019SyntheticEI}
S.~Beery, Y.~Liu, D.~Morris, J.~Piavis, A.~Kapoor, M.~Meister, and P.~Perona.
\newblock Synthetic examples improve generalization for rare classes.
\newblock {\em ArXiv}, abs/1904.05916, 2019.

\bibitem{berthelot2019mixmatch}
D.~Berthelot, N.~Carlini, I.~Goodfellow, N.~Papernot, A.~Oliver, and C.~A.
  Raffel.
\newblock Mixmatch: A holistic approach to semi-supervised learning.
\newblock In {\em Advances in Neural Information Processing Systems}, pages
  5050--5060, 2019.

\bibitem{Bissoto2019DeCB}
A.~Bissoto, M.~Fornaciali, E.~Valle, and S.~Avila.
\newblock (de) constructing bias on skin lesion datasets.
\newblock {\em 2019 IEEE/CVF Conference on Computer Vision and Pattern
  Recognition Workshops (CVPRW)}, pages 2766--2774, 2019.

\bibitem{bowles2018gan}
C.~Bowles, L.~Chen, R.~Guerrero, P.~Bentley, R.~Gunn, A.~Hammers, D.~A. Dickie,
  M.~V. Hern{\'a}ndez, J.~Wardlaw, and D.~Rueckert.
\newblock Gan augmentation: Augmenting training data using generative
  adversarial networks.
\newblock {\em arXiv preprint arXiv:1810.10863}, 2018.

\bibitem{Brock2016NeuralPE}
A.~Brock, T.~Lim, J.~M. Ritchie, and N.~Weston.
\newblock Neural photo editing with introspective adversarial networks.
\newblock {\em ArXiv}, abs/1609.07093, 2016.

\bibitem{Sun2019UnlabeledSG}
W.~chen Sun, F.~Liu, and W.~Xu.
\newblock Unlabeled samples generated by gan improve the person
  re-identification baseline.
\newblock In {\em ICCTA 2019}, 2019.

\bibitem{choi2018stargan}
Y.~Choi, M.~Choi, M.~Kim, J.-W. Ha, S.~Kim, and J.~Choo.
\newblock Stargan: Unified generative adversarial networks for multi-domain
  image-to-image translation.
\newblock In {\em Proceedings of the IEEE conference on computer vision and
  pattern recognition}, pages 8789--8797, 2018.

\bibitem{choi2020stargan}
Y.~Choi, Y.~Uh, J.~Yoo, and J.-W. Ha.
\newblock Stargan v2: Diverse image synthesis for multiple domains.
\newblock In {\em Proceedings of the IEEE/CVF Conference on Computer Vision and
  Pattern Recognition}, pages 8188--8197, 2020.

\bibitem{codella2018skin}
N.~C. Codella, D.~Gutman, M.~E. Celebi, B.~Helba, M.~A. Marchetti, S.~W. Dusza,
  A.~Kalloo, K.~Liopyris, N.~Mishra, H.~Kittler, et~al.
\newblock Skin lesion analysis toward melanoma detection: A challenge at the
  2017 international symposium on biomedical imaging (isbi), hosted by the
  international skin imaging collaboration (isic).
\newblock In {\em 2018 IEEE 15th International Symposium on Biomedical Imaging
  (ISBI 2018)}, pages 168--172. IEEE, 2018.

\bibitem{cubuk2019autoaugment}
E.~D. Cubuk, B.~Zoph, D.~Mane, V.~Vasudevan, and Q.~V. Le.
\newblock Autoaugment: Learning augmentation strategies from data.
\newblock In {\em Proceedings of the IEEE conference on computer vision and
  pattern recognition}, pages 113--123, 2019.

\bibitem{cubuk2019randaugment}
E.~D. Cubuk, B.~Zoph, J.~Shlens, and Q.~V. Le.
\newblock Randaugment: Practical data augmentation with no separate search.
\newblock {\em arXiv preprint arXiv:1909.13719}, 2019.

\bibitem{Cui2015DataAF}
X.~Cui, V.~Goel, and B.~Kingsbury.
\newblock Data augmentation for deep neural network acoustic modeling.
\newblock {\em IEEE/ACM Transactions on Audio, Speech, and Language
  Processing}, 23:1469--1477, 2015.

\bibitem{Dao2018AKT}
T.~Dao, A.~Gu, A.~J. Ratner, V.~Smith, C.~D. Sa, and C.~R{\'e}.
\newblock A kernel theory of modern data augmentation.
\newblock {\em Proceedings of machine learning research}, 97:1528--1537, 2018.

\bibitem{Deschacht2009SemisupervisedSR}
K.~Deschacht and M.-F. Moens.
\newblock Semi-supervised semantic role labeling using the latent words
  language model.
\newblock In {\em EMNLP}, 2009.

\bibitem{devries2017improved}
T.~DeVries and G.~W. Taylor.
\newblock Improved regularization of convolutional neural networks with cutout.
\newblock {\em arXiv preprint arXiv:1708.04552}, 2017.

\bibitem{Dvornik2018OnTI}
N.~Dvornik, J.~Mairal, and C.~Schmid.
\newblock On the importance of visual context for data augmentation in scene
  understanding.
\newblock {\em IEEE transactions on pattern analysis and machine intelligence},
  2018.

\bibitem{dwibedi2017cut}
D.~Dwibedi, I.~Misra, and M.~Hebert.
\newblock Cut, paste and learn: Surprisingly easy synthesis for instance
  detection.
\newblock In {\em Proceedings of the IEEE International Conference on Computer
  Vision}, pages 1301--1310, 2017.

\bibitem{Engstrom2017ARA}
L.~Engstrom, D.~Tsipras, L.~Schmidt, and A.~Madry.
\newblock A rotation and a translation suffice: Fooling cnns with simple
  transformations.
\newblock {\em ArXiv}, abs/1712.02779, 2017.

\bibitem{Fadaee2017DataAF}
M.~Fadaee, A.~Bisazza, and C.~Monz.
\newblock Data augmentation for low-resource neural machine translation.
\newblock In {\em ACL}, 2017.

\bibitem{ganin2016domain}
Y.~Ganin, E.~Ustinova, H.~Ajakan, P.~Germain, H.~Larochelle, F.~Laviolette,
  M.~Marchand, and V.~Lempitsky.
\newblock Domain-adversarial training of neural networks.
\newblock {\em The Journal of Machine Learning Research}, 17(1):2096--2030,
  2016.

\bibitem{Gardner2015DeepMT}
J.~R. Gardner, M.~J. Kusner, Y.~Li, P.~Upchurch, K.~Q. Weinberger, and J.~E.
  Hopcroft.
\newblock Deep manifold traversal: Changing labels with convolutional features.
\newblock {\em ArXiv}, abs/1511.06421, 2015.

\bibitem{gan}
I.~Goodfellow, J.~Pouget-Abadie, M.~Mirza, B.~Xu, D.~Warde-Farley, S.~Ozair,
  A.~Courville, and Y.~Bengio.
\newblock Generative adversarial nets.
\newblock In {\em Advances in neural information processing systems}, pages
  2672--2680, 2014.

\bibitem{Goodfellow2014ExplainingAH}
I.~J. Goodfellow, J.~Shlens, and C.~Szegedy.
\newblock Explaining and harnessing adversarial examples.
\newblock {\em CoRR}, abs/1412.6572, 2014.

\bibitem{gowal2019achieving}
S.~Gowal, C.~Qin, P.-S. Huang, T.~Cemgil, K.~Dvijotham, T.~Mann, and P.~Kohli.
\newblock Achieving robustness in the wild via adversarial mixing with
  disentangled representations.
\newblock {\em arXiv preprint arXiv:1912.03192}, 2019.

\bibitem{he2016identity}
K.~He, X.~Zhang, S.~Ren, and J.~Sun.
\newblock Identity mappings in deep residual networks.
\newblock In {\em European conference on computer vision}, pages 630--645.
  Springer, 2016.

\bibitem{heinze2017conditional}
C.~Heinze-Deml and N.~Meinshausen.
\newblock Conditional variance penalties and domain shift robustness.
\newblock {\em arXiv preprint arXiv:1710.11469}, 2017.

\bibitem{hendrycks2019augmix}
D.~Hendrycks, N.~Mu, E.~D. Cubuk, B.~Zoph, J.~Gilmer, and B.~Lakshminarayanan.
\newblock Augmix: A simple data processing method to improve robustness and
  uncertainty.
\newblock {\em arXiv preprint arXiv:1912.02781}, 2019.

\bibitem{ho2019population}
D.~Ho, E.~Liang, I.~Stoica, P.~Abbeel, and X.~Chen.
\newblock Population based augmentation: Efficient learning of augmentation
  policy schedules.
\newblock {\em arXiv preprint arXiv:1905.05393}, 2019.

\bibitem{Hu2019LearningDM}
Z.~Hu, B.~Tan, R.~Salakhutdinov, T.~M. Mitchell, and E.~P. Xing.
\newblock Learning data manipulation for augmentation and weighting.
\newblock In {\em NeurIPS}, 2019.

\bibitem{Huang2018AugGANCD}
S.-W. Huang, C.-T. Lin, S.-P. Chen, Y.-Y. Wu, P.-H. Hsu, and S.-H. Lai.
\newblock Auggan: Cross domain adaptation with gan-based data augmentation.
\newblock In {\em ECCV}, 2018.

\bibitem{isola2017image}
P.~Isola, J.-Y. Zhu, T.~Zhou, and A.~A. Efros.
\newblock Image-to-image translation with conditional adversarial networks.
\newblock In {\em Proceedings of the IEEE conference on computer vision and
  pattern recognition}, pages 1125--1134, 2017.

\bibitem{Jaitly2013VocalTL}
N.~Jaitly and E.~S. Hinton.
\newblock Vocal tract length perturbation (vtlp) improves speech recognition.
\newblock In {\em Proc. ICML Workshop on Deep Learning for Audio, Speech and
  Language}, 2013.

\bibitem{Jia2016DataRF}
R.~Jia and P.~Liang.
\newblock Data recombination for neural semantic parsing.
\newblock {\em ArXiv}, abs/1606.03622, 2016.

\bibitem{Kanbak2017GeometricRO}
C.~Kanbak, S.-M. Moosavi-Dezfooli, and P.~Frossard.
\newblock Geometric robustness of deep networks: Analysis and improvement.
\newblock {\em 2018 IEEE/CVF Conference on Computer Vision and Pattern
  Recognition}, pages 4441--4449, 2017.

\bibitem{kannan2018adversarial}
H.~Kannan, A.~Kurakin, and I.~Goodfellow.
\newblock Adversarial logit pairing.
\newblock {\em arXiv preprint arXiv:1803.06373}, 2018.

\bibitem{Karras2018ASG}
T.~Karras, S.~Laine, and T.~Aila.
\newblock A style-based generator architecture for generative adversarial
  networks.
\newblock {\em 2019 IEEE/CVF Conference on Computer Vision and Pattern
  Recognition (CVPR)}, pages 4396--4405, 2018.

\bibitem{Ko2015AudioAF}
T.~Ko, V.~Peddinti, D.~Povey, and S.~Khudanpur.
\newblock Audio augmentation for speech recognition.
\newblock In {\em INTERSPEECH}, 2015.

\bibitem{Kobayashi2018ContextualAD}
S.~Kobayashi.
\newblock Contextual augmentation: Data augmentation by words with paradigmatic
  relations.
\newblock {\em ArXiv}, abs/1805.06201, 2018.

\bibitem{Kolomiyets2011ModelPortabilityEF}
O.~Kolomiyets, S.~Bethard, and M.-F. Moens.
\newblock Model-portability experiments for textual temporal analysis.
\newblock In {\em ACL}, 2011.

\bibitem{Krizhevsky2012ImageNetCW}
A.~Krizhevsky, I.~Sutskever, and G.~E. Hinton.
\newblock Imagenet classification with deep convolutional neural networks.
\newblock In {\em NIPS}, 2012.

\bibitem{lecun1998gradient}
Y.~LeCun, L.~Bottou, Y.~Bengio, and P.~Haffner.
\newblock Gradient-based learning applied to document recognition.
\newblock {\em Proceedings of the IEEE}, 86(11):2278--2324, 1998.

\bibitem{li2018deep}
Y.~Li, X.~Tian, M.~Gong, Y.~Liu, T.~Liu, K.~Zhang, and D.~Tao.
\newblock Deep domain generalization via conditional invariant adversarial
  networks.
\newblock In {\em Proceedings of the European Conference on Computer Vision
  (ECCV)}, pages 624--639, 2018.

\bibitem{lim2019fast}
S.~Lim, I.~Kim, T.~Kim, C.~Kim, and S.~Kim.
\newblock Fast autoaugment.
\newblock In {\em Advances in Neural Information Processing Systems}, pages
  6662--6672, 2019.

\bibitem{liu2015deep}
Z.~Liu, P.~Luo, X.~Wang, and X.~Tang.
\newblock Deep learning face attributes in the wild.
\newblock In {\em Proceedings of the IEEE international conference on computer
  vision}, pages 3730--3738, 2015.

\bibitem{long2018conditional}
M.~Long, Z.~Cao, J.~Wang, and M.~I. Jordan.
\newblock Conditional adversarial domain adaptation.
\newblock In {\em Advances in Neural Information Processing Systems}, pages
  1640--1650, 2018.

\bibitem{Madry2017TowardsDL}
A.~Madry, A.~Makelov, L.~Schmidt, D.~Tsipras, and A.~Vladu.
\newblock Towards deep learning models resistant to adversarial attacks.
\newblock {\em ArXiv}, abs/1706.06083, 2017.

\bibitem{mariani2018bagan}
G.~Mariani, F.~Scheidegger, R.~Istrate, C.~Bekas, and C.~Malossi.
\newblock Bagan: Data augmentation with balancing gan.
\newblock {\em arXiv preprint arXiv:1803.09655}, 2018.

\bibitem{Mazzone2019ArtCA}
M.~Mazzone and A.~Elgammal.
\newblock Art, creativity, and the potential of artificial intelligence.
\newblock In {\em Arts}, volume~8, page~26. Multidisciplinary Digital
  Publishing Institute, 2019.

\bibitem{Molano2018GenerativeMF}
J.~M. Molano, R.~Paredes, and D.~Ramos-Castro.
\newblock Generative models for deep learning with very scarce data.
\newblock In {\em CIARP}, 2018.

\bibitem{MoosaviDezfooli2018RobustnessVC}
S.-M. Moosavi-Dezfooli, A.~Fawzi, J.~Uesato, and P.~Frossard.
\newblock Robustness via curvature regularization, and vice versa.
\newblock {\em 2019 IEEE/CVF Conference on Computer Vision and Pattern
  Recognition (CVPR)}, pages 9070--9078, 2018.

\bibitem{Mounsaveng2019AdversarialLO}
S.~Mounsaveng, D.~V{\'a}zquez, I.~B. Ayed, and M.~Pedersoli.
\newblock Adversarial learning of general transformations for data
  augmentation.
\newblock {\em ArXiv}, abs/1909.09801, 2019.

\bibitem{mu2019mnist}
N.~Mu and J.~Gilmer.
\newblock Mnist-c: A robustness benchmark for computer vision.
\newblock {\em arXiv preprint arXiv:1906.02337}, 2019.

\bibitem{Odena2016ConditionalIS}
A.~Odena, C.~Olah, and J.~Shlens.
\newblock Conditional image synthesis with auxiliary classifier gans.
\newblock In {\em ICML}, 2016.

\bibitem{Papernot2015DistillationAA}
N.~Papernot, P.~D. McDaniel, X.~Wu, S.~Jha, and A.~Swami.
\newblock Distillation as a defense to adversarial perturbations against deep
  neural networks.
\newblock {\em 2016 IEEE Symposium on Security and Privacy (SP)}, pages
  582--597, 2015.

\bibitem{Pesteie2019AdaptiveAO}
M.~Pesteie, P.~Abolmaesumi, and R.~Rohling.
\newblock Adaptive augmentation of medical data using independently conditional
  variational auto-encoders.
\newblock {\em IEEE Transactions on Medical Imaging}, 38:2807--2820, 2019.

\bibitem{Qiu2019SemanticAdvGA}
H.~Qiu, C.~Xiao, L.~Yang, X.~Yan, H.~Lee, and B.~Li.
\newblock Semanticadv: Generating adversarial examples via
  attribute-conditional image editing.
\newblock {\em ArXiv}, abs/1906.07927, 2019.

\bibitem{ratner2017learning}
A.~J. Ratner, H.~Ehrenberg, Z.~Hussain, J.~Dunnmon, and C.~R{\'e}.
\newblock Learning to compose domain-specific transformations for data
  augmentation.
\newblock In {\em Advances in neural information processing systems}, pages
  3236--3246, 2017.

\bibitem{Ratner2017LearningTC}
A.~J. Ratner, H.~R. Ehrenberg, Z.~Hussain, J.~Dunnmon, and C.~R{\'e}.
\newblock Learning to compose domain-specific transformations for data
  augmentation.
\newblock {\em Advances in neural information processing systems},
  30:3239--3249, 2017.

\bibitem{Reed2014LearningTD}
S.~E. Reed, K.~Sohn, Y.~Zhang, and H.~Lee.
\newblock Learning to disentangle factors of variation with manifold
  interaction.
\newblock In {\em ICML}, 2014.

\bibitem{Reed2015DeepVA}
S.~E. Reed, Y.~Zhang, Y.~Zhang, and H.~Lee.
\newblock Deep visual analogy-making.
\newblock In {\em NIPS}, 2015.

\bibitem{Rieger2019InterpretationsAU}
L.~Rieger, C.~Singh, W.~J. Murdoch, and B.~Yu.
\newblock Interpretations are useful: penalizing explanations to align neural
  networks with prior knowledge.
\newblock {\em ArXiv}, abs/1909.13584, 2019.

\bibitem{sagawa2019distributionally}
S.~Sagawa, P.~W. Koh, T.~B. Hashimoto, and P.~Liang.
\newblock Distributionally robust neural networks for group shifts: On the
  importance of regularization for worst-case generalization.
\newblock {\em arXiv preprint arXiv:1911.08731}, 2019.

\bibitem{Sandfort2019DataAU}
V.~Sandfort, K.~Yan, P.~J. Pickhardt, and R.~M. Summers.
\newblock Data augmentation using generative adversarial networks (cyclegan) to
  improve generalizability in ct segmentation tasks.
\newblock In {\em Scientific Reports}, 2019.

\bibitem{selvaraju2017grad}
R.~R. Selvaraju, M.~Cogswell, A.~Das, R.~Vedantam, D.~Parikh, and D.~Batra.
\newblock Grad-cam: Visual explanations from deep networks via gradient-based
  localization.
\newblock In {\em Proceedings of the IEEE international conference on computer
  vision}, pages 618--626, 2017.

\bibitem{Sennrich2015ImprovingNM}
R.~Sennrich, B.~Haddow, and A.~Birch.
\newblock Improving neural machine translation models with monolingual data.
\newblock {\em ArXiv}, abs/1511.06709, 2015.

\bibitem{Silfverberg2017DataAF}
M.~Silfverberg, A.~Wiemerslage, L.~Liu, and L.~J. Mao.
\newblock Data augmentation for morphological reinflection.
\newblock In {\em CoNLL Shared Task}, 2017.

\bibitem{Simard1992EfficientPR}
P.~Y. Simard, Y.~LeCun, and J.~S. Denker.
\newblock Efficient pattern recognition using a new transformation distance.
\newblock In {\em NIPS}, 1992.

\bibitem{Simard1998TransformationII}
P.~Y. Simard, Y.~LeCun, J.~S. Denker, and B.~Victorri.
\newblock Transformation invariance in pattern recognition - tangent distance
  and tangent propagation.
\newblock In {\em Neural Networks: Tricks of the Trade}, 1998.

\bibitem{Simard2003BestPF}
P.~Y. Simard, D.~Steinkraus, and J.~C. Platt.
\newblock Best practices for convolutional neural networks applied to visual
  document analysis.
\newblock {\em Seventh International Conference on Document Analysis and
  Recognition, 2003. Proceedings.}, pages 958--963, 2003.

\bibitem{Simard1991TangentP}
P.~Y. Simard, B.~Victorri, Y.~LeCun, and J.~S. Denker.
\newblock Tangent prop - a formalism for specifying selected invariances in an
  adaptive network.
\newblock In {\em NIPS}, 1991.

\bibitem{Song2018ConstructingUA}
Y.~Song, R.~Shu, N.~Kushman, and S.~Ermon.
\newblock Constructing unrestricted adversarial examples with generative
  models.
\newblock In {\em NeurIPS}, 2018.

\bibitem{Stylianou1998ContinuousPT}
Y.~Stylianou, O.~Capp{\'e}, and E.~Moulines.
\newblock Continuous probabilistic transform for voice conversion.
\newblock {\em IEEE Trans. Speech and Audio Processing}, 6:131--142, 1998.

\bibitem{Szegedy2014GoingDW}
C.~Szegedy, W.~Liu, Y.~Jia, P.~Sermanet, S.~Reed, D.~Anguelov, D.~Erhan,
  V.~Vanhoucke, and A.~Rabinovich.
\newblock Going deeper with convolutions.
\newblock {\em 2015 IEEE Conference on Computer Vision and Pattern Recognition
  (CVPR)}, pages 1--9, 2014.

\bibitem{Szegedy2013IntriguingPO}
C.~Szegedy, W.~Zaremba, I.~Sutskever, J.~Bruna, D.~Erhan, I.~J. Goodfellow, and
  R.~Fergus.
\newblock Intriguing properties of neural networks.
\newblock {\em CoRR}, abs/1312.6199, 2013.

\bibitem{Tran2017ABD}
T.~Tran, T.~Pham, G.~Carneiro, L.~J. Palmer, and I.~D. Reid.
\newblock A bayesian data augmentation approach for learning deep models.
\newblock {\em ArXiv}, abs/1710.10564, 2017.

\bibitem{upchurch2017deep}
P.~Upchurch, J.~Gardner, G.~Pleiss, R.~Pless, N.~Snavely, K.~Bala, and
  K.~Weinberger.
\newblock Deep feature interpolation for image content changes.
\newblock In {\em Proceedings of the IEEE conference on computer vision and
  pattern recognition}, pages 7064--7073, 2017.

\bibitem{Wang2015ThatsSA}
W.~Y. Wang and D.~Yang.
\newblock That's so annoying!!!: A lexical and frame-semantic embedding based
  data augmentation approach to automatic categorization of annoying behaviors
  using petpeeve tweets.
\newblock In {\em EMNLP}, 2015.

\bibitem{wang2018transferring}
Y.~Wang, C.~Wu, L.~Herranz, J.~van~de Weijer, A.~Gonzalez-Garcia, and
  B.~Raducanu.
\newblock Transferring gans: generating images from limited data.
\newblock In {\em Proceedings of the European Conference on Computer Vision
  (ECCV)}, pages 218--234, 2018.

\bibitem{Wei2019EDAED}
J.~Wei and K.~Zou.
\newblock Eda: Easy data augmentation techniques for boosting performance on
  text classification tasks.
\newblock In {\em EMNLP/IJCNLP}, 2019.

\bibitem{winkler2019association}
J.~K. Winkler, C.~Fink, F.~Toberer, A.~Enk, T.~Deinlein, R.~Hofmann-Wellenhof,
  L.~Thomas, A.~Lallas, A.~Blum, W.~Stolz, et~al.
\newblock Association between surgical skin markings in dermoscopic images and
  diagnostic performance of a deep learning convolutional neural network for
  melanoma recognition.
\newblock {\em JAMA dermatology}, 155(10):1135--1141, 2019.

\bibitem{Xiao2018GeneratingAE}
C.~Xiao, B.~Li, J.-Y. Zhu, W.~He, M.~Liu, and D.~X. Song.
\newblock Generating adversarial examples with adversarial networks.
\newblock In {\em IJCAI}, 2018.

\bibitem{xie2019unsupervised}
Q.~Xie, Z.~Dai, E.~Hovy, M.-T. Luong, and Q.~V. Le.
\newblock Unsupervised data augmentation for consistency training.
\newblock {\em arXiv preprint arXiv:1904.12848}, 2019.

\bibitem{Xie2017DataNA}
Z.~Xie, S.~I. Wang, J.~Li, D.~L{\'e}vy, A.~Nie, D.~Jurafsky, and A.~Y. Ng.
\newblock Data noising as smoothing in neural network language models.
\newblock {\em ArXiv}, abs/1703.02573, 2017.

\bibitem{Yaeger1996EffectiveTO}
L.~S. Yaeger, R.~F. Lyon, and B.~J. Webb.
\newblock Effective training of a neural network character classifier for word
  recognition.
\newblock In {\em NIPS}, 1996.

\bibitem{Yu2018QANetCL}
A.~W. Yu, D.~Dohan, M.-T. Luong, R.~Zhao, K.~Chen, M.~Norouzi, and Q.~V. Le.
\newblock Qanet: Combining local convolution with global self-attention for
  reading comprehension.
\newblock {\em ArXiv}, abs/1804.09541, 2018.

\bibitem{yun2019cutmix}
S.~Yun, D.~Han, S.~J. Oh, S.~Chun, J.~Choe, and Y.~Yoo.
\newblock Cutmix: Regularization strategy to train strong classifiers with
  localizable features.
\newblock In {\em Proceedings of the IEEE International Conference on Computer
  Vision}, pages 6023--6032, 2019.

\bibitem{zhang2017mixup}
H.~Zhang, M.~Cisse, Y.~N. Dauphin, and D.~Lopez-Paz.
\newblock mixup: Beyond empirical risk minimization.
\newblock {\em arXiv preprint arXiv:1710.09412}, 2017.

\bibitem{Zhang2018DADADA}
X.~Zhang, Z.~Wang, D.~Liu, and Q.~Ling.
\newblock Dada: Deep adversarial data augmentation for extremely low data
  regime classification.
\newblock {\em ICASSP 2019 - 2019 IEEE International Conference on Acoustics,
  Speech and Signal Processing (ICASSP)}, pages 2807--2811, 2018.

\bibitem{Zhang2015CharacterlevelCN}
X.~Zhang, J.~J. Zhao, and Y.~LeCun.
\newblock Character-level convolutional networks for text classification.
\newblock In {\em NIPS}, 2015.

\bibitem{zheng2016improving}
S.~Zheng, Y.~Song, T.~Leung, and I.~Goodfellow.
\newblock Improving the robustness of deep neural networks via stability
  training.
\newblock In {\em Proceedings of the ieee conference on computer vision and
  pattern recognition}, pages 4480--4488, 2016.

\bibitem{zhu2017unpaired}
J.-Y. Zhu, T.~Park, P.~Isola, and A.~A. Efros.
\newblock Unpaired image-to-image translation using cycle-consistent
  adversarial networks.
\newblock In {\em Proceedings of the IEEE international conference on computer
  vision}, pages 2223--2232, 2017.

\bibitem{Zoph2019LearningDA}
B.~Zoph, E.~D. Cubuk, G.~Ghiasi, T.-Y. Lin, J.~Shlens, and Q.~V. Le.
\newblock Learning data augmentation strategies for object detection.
\newblock {\em ArXiv}, abs/1906.11172, 2019.

\end{thebibliography}
